\documentclass[11pt]{article}

\usepackage[final]{acl}

\usepackage{times}
\usepackage{latexsym}

\usepackage[T1]{fontenc}

\usepackage[utf8]{inputenc}

\usepackage{microtype}

\usepackage{inconsolata}

\usepackage{graphicx}
\usepackage{subcaption}
\usepackage{url}            
\usepackage{booktabs}       
\usepackage{nicefrac}       
\usepackage{xcolor}         

\usepackage{amsmath}
\usepackage{amssymb}
\usepackage{mathtools}
\usepackage{amsthm}
\usepackage{comment}
\usepackage{enumitem}
\usepackage{amsfonts}

\usepackage{tcolorbox}              
\usepackage{algorithm}
\usepackage{adjustbox}
\usepackage{bbm}
\usepackage{pifont}
\usepackage{multirow}
\usepackage[capitalize,noabbrev]{cleveref}
\usepackage[table]{xcolor}

\usepackage{algpseudocode}

\crefname{equation}{Eq.}{Eqs.}
\Crefname{equation}{Equation}{Equations}

\crefname{figure}{Fig.}{Figs.}
\Crefname{figure}{Figure}{Figures}

\crefname{table}{Tab.}{Tabs.}
\Crefname{table}{Table}{Tables}

\newcommand{\topic}[1]{\noindent\textbf{#1}}

\usepackage{xspace}
\newcommand{\shortname}{\textsc{Q-Skew}\xspace}

\theoremstyle{plain}
\newtheorem{theorem}{Theorem}[section]

\theoremstyle{definition}

\newtheorem{assumption}[theorem]{Assumption}
\theoremstyle{remark}
\newtheorem{remark}[theorem]{Remark}

\title{%
    Membership Inference in Fine-tuned Diffusion Language Models \\via Token-level Memorization Asymmetry}

\author{
Shengfang Zhai\textsuperscript{1}%
\thanks{Correspondence: shengfang.zhai@nus.edu.sg}, \,
Leo Marchyok\textsuperscript{2}, \,
Yuling Shi\textsuperscript{3}, \,
Huanran Chen\textsuperscript{4}, 
\\
\textbf{
Yinpeng Dong\textsuperscript{4}, \,
Jiaheng Zhang\textsuperscript{1}, \,
Sanghyun Hong\textsuperscript{2} 
}
\\[0.5ex]
\textsuperscript{1} National University of Singapore,
\textsuperscript{2} Oregon State University \\
\textsuperscript{3} Shanghai Jiao Tong University,
\textsuperscript{4} College of AI, Tsinghua University 
}

\begin{document}
\maketitle

\begin{abstract}

Diffusion language models (DLMs) have recently emerged as an alternative modeling paradigm to autoregressive LMs, offering advantages such as parallel generation and bidirectional context modeling.
Despite growing interest in their generative capabilities, the privacy risks of DLMs remain underexplored.
We identify a phenomenon termed \emph{token-level memorization asymmetry} through theoretical analysis of diffusion training dynamics.
Building on this finding, we propose \shortname,
a quantile-weighted skewness-based indicator for membership inference on finetuned DLMs.
Experiments across multiple fine-tuning datasets and models show that our method outperforms existing baselines.
Moreover, we show that \shortname can also facilitate 
other privacy violations, such as PII extraction attacks.
Our findings reveal a previously underexplored privacy attack surface and highlight the need for systematic privacy evaluation of DLMs.
\end{abstract}

\section{Introduction}
\label{sec:intro}

Diffusion language models (DLMs)~\cite{nie2025large,bie2025llada2, ye2025dream} are emerging as a compelling alternative to autoregressive (AR) generation. 
Recent advances demonstrate that diffusion can be successfully adapted from continuous domains to text, leading to practical systems such as Gemini Diffusion~\cite{gemini_diffusion_2025}, SEED Diffusion~\cite{song2025seed}, and a growing body of work from both academia~\cite{nie2025large} and industry~\cite{inception_mercury_chat_2025}. 
Unlike AR models, which strictly generate tokens sequentially, DLMs generate text through iterative denoising processes that refine entire sequences over multiple steps. 
This paradigm offers capabilities that AR modeling does not naturally support, including flexible decoding~\cite{de2025accelerated}, global sequence refinement~\cite{nie2025large}, and improved controllability over generation dynamics~\cite{li2022diffusion}, making diffusion an attractive generation paradigm other than next-token generation.

Driven by ongoing research efforts, state-of-the-art DLMs gradually achieve generative capabilities comparable to those of AR models, both in open-source models~\cite{bie2025llada2} and in commercial SaaS services~\cite{inception_mercury_chat_2025}.
However, few studies investigate the privacy issues of DLMs as a new language model architecture.
In this paper, we mainly focus on Membership Inference (MI), which aims to determine whether a given data point is used to train the target model. 
On one hand, a malicious adversary can leverage this technique to violate the model privacy~\cite{shokri2017membership}. 
On the other hand, membership inference can also be used for unauthorized data auditing~\cite{dealcala2024my}. 

A series of studies have developed membership inference methods against common (autoregressive) language models~\cite{yeom2018privacy}, or the (vision) diffusion models. 
However, applying these methods to diffusion language models poses significant challenges.
First, DLMs adopt an any-order denoising training objective rather than a single fixed left-to-right decomposition, which acts as an implicit Monte Carlo data augmentation mechanism and reduces exact sample memorization~\cite{ni2025diffusion}.
Second, the diffusion training process introduces randomness for each token, since different tokens are masked with different probabilities. 
This fundamentally breaks methods based on lower-tail token scores, such as MinK~\cite{shi2023detecting} and Min-K++~\cite{zhang2024min}.
Third, membership inference methods on vision diffusion models~\cite{ho2020denoising} show degraded performance on DLMs due to the transition of the state space from continuous to discrete.

To address this challenge, we conduct an in-depth analysis of diffusion training dynamics and find that the model's single-step memorization gain for each token is inversely proportional to the mask ratio during diffusion training process. 
Based on this, we theoretically derive the property of token-level memorization asymmetry and accordingly propose a skewness-based MI indicator: \textbf{Q}uantile-weighted \textbf{Skew}ness (\shortname) of the sample tokens.
Our method can effectively distinguish member sets from non-member sets from a data distribution perspective, beyond the memorization at the individual sample level.
We conduct experiments across different finetuning objectives, dataset domains, and types of diffusion language models, including both base and instruction-tuned diffusion language models. 
\shortname outperforms baselines on average, and this superiority remains consistent across different training epochs, with an average AUC improvement of more than 10\%. 
Furthermore, we show that the skewness indicator can also enhance the effectiveness of PII (Personally Identifiable Information) extraction.

Notably, there is a concurrent work SAMA \cite{chen2026membership} which also focuses on membership inference.
Our work differs in two aspects: (1) we propose a new indicator beyond memorization at the individual sample level, and (2) our evaluation covers both instruction fine-tuning and domain fine-tuning (same training objective with pretraining), rather than only domain fine-tuning as in prior works~\cite{fu2024membership,chen2026membership}. Our method consistently outperforms theirs in both settings.

In summary, our contributions are: 
\begin{itemize}[itemsep=0.1em, topsep=0.2em, leftmargin=1.4em]
    \item We are the first to identify an inverse relationship 
    between mask ratio and token memory gain, 
    and theoretically characterize Token-level Memorization Asymmetry. 
    \item  We propose \shortname, 
    a novel membership inference 
    indicator based on the inverse mask-ratio phenomenon
    that outperforms baselines across 
    datasets and training settings.
    \item We provide a new perspective on privacy risks in DLMs,
    showing that skewness also improves PII extraction
    and highlighting fundamental differences from AR models.
\end{itemize}

\section{Background and Related Work}
\label{sec:background}

\subsection{Diffusion Language Models}
\label{subsec:dllms}


Unlike traditional autoregressive generation, diffusion language models (DLMs) 
recover data from a fully masked state through iterative denoising.
Currently, successful large-scale DLMs are based on the discrete masked diffusion model architecture. Taking this as an example, we formalize the training and inference processes of DLMs as follows. 
Let $\mathbf{x}_0 = (x_0^1, x_0^2, \dots, x_0^L)$ denotes a sample of the training set $\mathcal{D}_\text{train}$, consisting of $L$ discrete tokens, where  $x_0^i \in \mathcal{V}$.
The diffusion process is defined on an extended vocabulary  $\mathcal{V}^+ = \mathcal{V} \cup \{[\texttt{MASK}]\}$ that includes a mask token.
The forward diffusion process in DLMs is defined as a Markov chain that gradually replaces tokens with mask tokens. Then we can derive a closed-form solution to sample the state $\mathbf{x}_t$ directly at timestep $t$:
$$q(x_t^i | x_0^i) = \begin{cases} 
1 - \bar{\alpha}_t & \text{if } x_t^i = [\texttt{MASK}] \\
\bar{\alpha}_t & \text{if } x_t^i = x_0^i \\
0 & \text{otherwise},
\end{cases}$$
where $\bar{\alpha}_t$ denotes the cumulative probability that a token remains unmasked at timestep $t$.
Then the parameterized denoiser $p_\theta(\mathbf{x}_0 \mid \mathbf{x}_t)$ is trained using the following masked-position denoising objective:
\begin{equation}
    \small
    \mathcal{L}_{\text{DLM}}(\theta) = \mathbb{E}_{t, \mathbf{x}_0, \mathbf{x}_t} \left[ \frac{1}{|\mathcal{M}_t|} \sum_{i \in \mathcal{M}_t} -\log p_\theta(x_0^i | \mathbf{x}_t) \right],
\end{equation}
where $t \sim \mathcal{U}(1, T)$, $\mathbf{x}_0\sim\mathcal{D}_{\text{train}}$, and  $\mathcal{I}_t = \{i \mid x_t^i = [\texttt{MASK}]\}$ denotes indices of masked tokens in timestamp $t$.
During the inference process, the model $\theta$ starts from a fully masked sequence $\mathbf{x}_T$, and performs the following update steps through iterative sampling until $t=0$:
\begin{equation*}
    \mathbf{x}_{t-1} \sim p_\theta(\mathbf{x}_{t-1} | \mathbf{x}_t).  
\end{equation*}

\subsection{Membership Inference}
Membership inference aims to determine whether a given data sample belongs to the member set (training data) of a target model~\citep{shokri2017membership}. 
This task is widely regarded as a core method for quantifying privacy risks or auditing unauthorized data usage~\citep{dealcala2024my}. 
For generative models, the high number of parameters results in significant fine-tuning overhead. Hence, existing effective membership inference mainly falls under query-based methods, which typically require a gray-box setting, i.e., obtaining output logits. 

\noindent\textbf{Membership inference on DMs.}  
Beyond the security concerns
\citep{zhai2023text,zhai2024discovering,zhai2025efficient,zhai2026baddlm}, privacy studies on DMs have primarily focused on membership inference (MI).
However, they mainly focus on visual (continuous) diffusion models~\citep{ho2020denoising}. In the grey-box setting, these methods examine the characteristics of diffusion models and design various metrics to determine member and non-member sets, such as: (1) estimation errors of deterministic forward process~\citep{duan2023diffusion,kong2024an}, (2) probability fluctuations in the presence of perturbations~\citep{fu2024probabilisticfluctuationbasedmembership}; (3) likelihood discrepancy with different inputs~\citep{zhai2024membership}, et al.

\noindent\textbf{Membership inference on LLMs.}  
Many studies explore membership inference for LLMs \citep{yeom2018privacy,shi2023detecting,zhang2024min}. In the fine-tuning scenario, due to the open availability of base models, methods for LLMs typically involve an additional assumption: the attacker has access to reference models. \citet{watson2022on} firstly introduces this assumption and proposes \textit{Calibration Attack}.
SPV-MIA~\cite{fu2024membership} leverages self-prompting to construct reference models, leading to more precise inference.  
Similarly, the concurrent work, SAMA~\citep{chen2026membership}, conducts membership inference on diffusion language models (DLMs) following the same setting. In contrast to our work, their empirical approach lacks theoretical analysis of the difference between DLMs and traditional AR language models, leading to less effectiveness in practice (\cref{table:main_exp_dft}, \cref{table:main_exp_ift}).

\section{Methodology}
\label{sec:membership-inference}

\subsection{Threat Model}
\label{subsec:threat}

In this work, we consider an adversary that aims to infer whether a given data record was used to train the target model $\theta_{\text{tar}}$. Since the pre-training process primarily utilizes crawled public data and is implemented by only a few third-party institutions, we focus on the fine-tuning stage, which is generally recognized as the stage that is most susceptible to privacy or copyright risks~\cite{yudifferentially, fu2024membership,chen2026membership}.

We strictly follow the mainstream setting of query-based membership inference~\cite {shi2023detecting,zhang2024min,duan2023diffusion,watson2022on}, where the adversary can only access the output logits of the target model $\theta_\text{tar}$, without having access to the model weights or gradients. The formalization is:
$$
\mathcal{A}(\mathbf{x},\theta_\text{tar}) = \mathbbm{1} \left[
\mathcal{A}'(\mathbf{x},\theta_\text{tar})
> \tau  \right],
$$
where $\mathcal{A}'$ denotes an indicator function that reflects membership information, and $\tau$ denotes a tunable decision threshold. 
Following the state-of-the-art reference-based membership inference methods~\cite{meng2025rr,watson2022on,fu2024membership,huang2025df}, we further assume that the adversary has access to an unfinetuned reference model $\theta_{\text{ref}}$. Then, a more precise strategy of membership inference can be formalized as:
$$
\mathcal{A}(\mathbf{x},\theta_\text{tar}) = \mathbbm{1} \left[
\mathcal{A}'(\mathbf{x},\theta_\text{tar},\theta_\text{ref})
> \tau  \right].
$$
Note that this setup is easily achievable for membership inference on fine-tuning scenarios due to the availability of open-sourced base models. And we further consider a weaker assumption involving misaligned reference models in \cref{sec:weaker}.

\noindent\textbf{Fine-tuning setups.}
Notably, previous membership inference methods on fine-tuned language models ~\citep{fu2024membership, chen2026membership} consider only domain fine-tuning, which shares the same training objective as the pretraining stage and utilizes a fixed training block size, such as 64, 128, or 256 tokens. 
However, we emphasize that these evaluation setups are limited: \ding{182} they do not represent the fine-tuning scenarios for instruction-tuning, which constitute the primary fine-tuning paradigm. \ding{183} they may lead to an inflated illusion of membership inference success by failing to account for the variable-length nature of instruction-tuning datasets.
Therefore, in our experiments, we evaluate both \textit{domain fine-tuning (DFT)} and \textit{instruction fine-tuning (IFT)} scenarios to align with real-world scenarios (\cref{sec:evaluation}).

\subsection{Token-level Memorization Asymmetry}

\begin{figure}[t]
    \centering
    \includegraphics[width=1\linewidth]{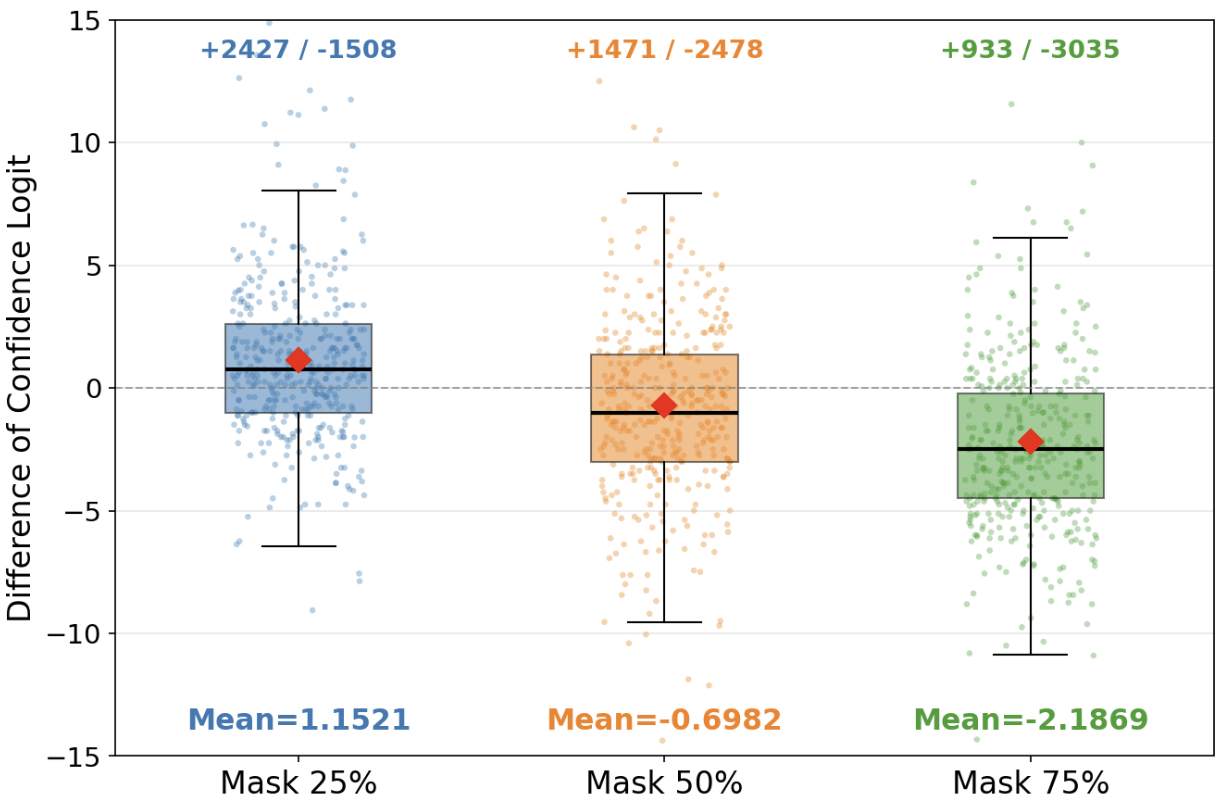}
    \caption{Violin and jitter plots of confidence logit differences before and after training. A smaller mask ratio results in a higher memorization gain for the token post-training.}
    \label{fig:memory_ratio}
\end{figure}

The training process of DLMs distinguishes them from AR language models and thereby introduces unique characteristics. 
In this part, we provide a formal proof characterizing the distributional differences of token-level memorization between member (training) data $\mathbf{x}\in \mathcal{D}_{train}$ and non-member (test) data $\mathbf{x}\in \mathcal{D}_{test}$. 

Let $\mathbf{x}$ be a sequence of tokens $\mathbf{x}=(w_1,\dots,w_L)$ in the training set $\mathcal{D}_{train}$. 
We consider a diffusion language model $\theta_K$ trained for $K$ epochs on $\mathcal{D}_{train}$, and $\theta_0$ denotes the original version that has not yet been trained.  
Since the model's memorization of training data gradually accumulates as the training progresses, we define the memorization gain $\mathcal{G}_{w_i}$ during the training process as the cumulative reduction in loss. 
For the $i$-th token $w_i$ in sequence $\mathbf{x}$, we have:
\begin{equation}
    \mathcal{G}_{w_i} := \mathcal{L}(w_i; \theta_0) - \mathcal{L}(w_i; \theta_K) = \sum_{k=1}^K \delta_{w_i,k},
\end{equation}
where the $\delta_{w_i,k}$ denotes the memory gain (i.e. marginal loss reduction) for token $w_i$ at epoch $k$.

\topic{Member set analysis.}
For the member sample $\mathbf{x}\in \mathcal{D}_{train}$, a global mask rate $\beta_k \sim U(0,1)$ is sampled at the training epoch $k$.
Let $m_{w_i,k} \in \{0,1\}$  denotes the binary indicator whether token $w_i$ is masked (and thus trained) at epoch $k$ \footnote{For clarity, we do not consider the impact of parameter updates during training on other masked tokens. In practice, parameter changes from training on other masked tokens may introduce noise that is unbiased on average, and thus does not affect our subsequent conclusions.
}:
$$P(m_{w_i,k}=1 \mid \beta_k) = \beta_k.$$
In the diffusion training progress, we find the increase in token memorization is associated with the masking ratio.
Insightfully, low-$\beta$ steps may lead to more concentrated gradient updates, resulting in stronger single-step memory gain for $w_i$.
Hence, we propose the following assumption:
\begin{assumption}(Inverse Mask-Ratio Scaling)
\label{theorem:assumption1}
We posit that the magnitude of the memory gain, when a token is selected, is a function of the mask rate:
$$\delta_{w_i,k} = m_{w_i,k} \cdot \phi(w_i, \beta_k),$$
where $\phi: (0,1) \to \mathbb{R}^+$ denotes the update strength function reflected by the loss and is strictly monotonically decreasing. 
\end{assumption}
To validate \cref{theorem:assumption1}, we fine-tuned LLaDA~\cite{nie2025large} on the XSUM \cite{narayan2018don} dataset for four epochs. During training, we set the mask ratio to three fixed values, 0.25, 0.50, and 0.75, with a fixed mask template for each data sample, instead of using a random mask ratio.
After training, we utilize the prediction confidence to estimate the memorization of individual tokens. 
It is evident from \cref{fig:memory_ratio} that a lower mask ratio leads to more significant token memorization on average under the same number of training epochs. 
We leave more experiments in \cref{sec:more_inverse} for validation.
Therefore, for a member datapoint $x$, the single-step gain $\delta_{w_i,k}$ obviously follows a zero-inflated mixture distribution.
We then derive the shape properties of the cumulative gain $\mathcal{G}$. 
\begin{theorem}
\label{th:pos_skew}
(Proof in \cref{appendix:prf})
Let $\mathcal{G}$ be the memorization gain for token-level in the training set. For any finite number of epochs $K \ge 1$, the distribution of $\mathcal{G}_{w_i}$ is strictly right-skewed:
\begin{equation}
\label{eq:skewness}
    \text{Skew}(\mathcal{G}_{w_i}) = \frac{\mathbb{E}\left[\left(\mathcal{G}_{w_i} - \mathbb{E}[\mathcal{G}_{w_i}]\right)^3\right]}{\sigma(\mathcal{G}_{w_i})^3} > 0
\end{equation}

\end{theorem}
We further analyze and demonstrate in~\cref{appendix:longtail} that this skewness specifically represents a long-tail distribution.

\topic{Non-member set analysis.}
For the non-member sample $\mathbf{x}\notin \mathcal{D}_{train}$,  $\mathcal{G}_{w_i} = \sum_{k=1}^K \epsilon_{i,k}$, where $\epsilon_{w_i,k}$ represents stochastic generalization noise. 
Without the zero-inflated selection mechanism and the monotonic inverse mask-ratio amplification assumed for members, the resulting distribution typically exhibits much weaker asymmetry.
In practice, this value is often close to zero when the aggregated fluctuations are approximately symmetric:
$\text{Skew}(\mathcal{G}_{w_i}) \approx 0$.
Recalling Eq. (1), we define the following for diffusion language models:
$$\mathcal{L}(w_i; \theta) := \mathbb{E}_{t, \mathbf{x}_0, \mathbf{x}_t} \left[ -\log p_\theta(w_i \mid \mathbf{x}_t) \right],$$
where $\mathbf{x}_t$ denote the noisy sequence of $\mathbf{x}$ at timestep $t$.
We use $\Delta_{w_i}$ as the empirical estimator of the memorization gain $\mathcal{G}_{w_i}$: 
\begin{equation}
{\small
    \begin{split}
        & \Delta_{w_i}(\theta_{\text{tar}}, \theta_{\text{ref}}) := \\ 
        & \underbrace{\mathbb{E}_{t, \mathbf{x}_t} \left[ -\log p_{\theta_{\text{ref}}}(w_i \mid \mathbf{x}_t) \right]}_{\text{Memorization of Reference Model}} - \underbrace{\mathbb{E}_{t, \mathbf{x}_t} \left[ -\log p_{\theta_{\text{tar}}}(w_i \mid \mathbf{x}_t) \right]}_{\text{Memorization of Target Model}}.
    \end{split}
}
\end{equation}
Assuming the reference model approximates the untrained state, i.e. $\theta_{\text{ref}} \approx \theta_0$, and the target model is the trained state, i.e.  $\theta_{\text{tar}} = \theta_K$ after $K$ epoch training, the $\Delta_{w_i}(\theta_{\text{tar}}, \theta_{\text{ref}})$ shares the same distribution properties (\cref{th:pos_skew}).
Hence, the following equation holds:
\[
{\small
\begin{aligned}
   & \text{Skew}(  \{ \Delta_{w_{\text{member},i}}(\theta_{\text{tar}}, \theta_{\text{ref}}) \mid i \in \{1, \dots, L\} \} ) \\
   & > 0 \approx \text{Skew}( \{ \Delta_{w_{\text{non-member},i}}(\theta_{\text{tar}}, \theta_{\text{ref}}) \mid i \in \{1, \dots, L\} \} )
\end{aligned}
}
\]
where $w_{\text{member},i}$ and $w_{\text{non-member}, i}$ denote the tokens of the member set and non-member set samples, respectively. We define:
\[ 
{
\small
\mathbb{I}(\mathbf{x}, \theta_{\text{tar}}, \theta_{\text{ref}}) = 
    \text{Skew}(  \{ \Delta_{w_i}(\theta_{\text{tar}}, \theta_{\text{ref}}) | i \in \{1, \dots, L\} ), 
    }
\]
where $\mathbf{x}= (w_1, \dots, w_L)$. For a given data sample $\mathbf{x}$, a reference model $\theta_{\text{ref}}$, and a target model $\theta_{\text{tar}}$, the indicator function $\mathbb{I}(\mathbf{x}, \theta_{\text{tar}}, \theta_{\text{ref}})$ serves as a metric for membership inference.
Intuitively, if $\mathbb{I}(\mathbf{x}, \theta_{\text{tar}}, \theta_{\text{ref}})$ exceeds a threshold $\tau$, the sample $\mathbf{x}$ is likely from the member set; otherwise, it belongs to the non-member set.

\subsection{Quantile-weighted Skewness with Cyclic Sampling}

\begin{table*}[t]
  \centering
  \caption{The main experiment results of membership inference on diffusion language models on domain fine-tuning tasks (DFT). We bold the best results, and our method is marked in grey.}
  \resizebox{\linewidth}{!}{
    \begin{tabular}{lccccccccc}
    \toprule
          & \multicolumn{3}{c}{Arxiv (DFT)} & \multicolumn{3}{c}{WikiText (DFT)} & \multicolumn{3}{c}{XSUM (DFT)} \\
          \cmidrule(lr){2-4} \cmidrule(lr){5-7} \cmidrule(lr){8-10}
          & AUC   & TPR@10\%FPR & TPR@1\%FPR & AUC   & TPR@10\%FPR & TPR@1\%FPR & AUC   & TPR@10\%FPR & TPR@1\%FPR \\
    \midrule
    Loss  &   61.3    &   16.4    &  5.0     &   60.3    & 17.1      & 2.0    &   61.1    &  15.8     & 1.7 \\
    Min-K\%  & 64.5 & 24.1  & 2.7   &    62.5    &   24.0    & 8.7   & 65.8 & 29.0  & 4.8  \\
    Min-K\%++ & 66.4 & 29.0  & 1.3   &      56.9  &   16.2    & 2.3     & 61.2 & 27.3  & 12.0 \\
    Calibration & 59.1 & 22.7  & 5.3   &    56.9   &  19.7     &   2.0    & 56.4 & 16.8  & 2.0 \\
    $\text{SecMI}_\text{DLM}$ &  58.0     &  18.8     &   2.0    &   62.4     &  26.0     & 1.7 &   64.8   &   31.7    & 17.3 \\
    SAMA  & 70.2 & 34.5  & 17.6  &      70.3    &    36.4   &  17.2   & 71.0 & 40.5  & 18.2 \\
\rowcolor[gray]{0.9} \textbf{Q-Skew} & \textbf{83.4} & \textbf{66.1} & \textbf{50.6} &      \textbf{80.3}     & \textbf{57.7} & \textbf{39.7}    & \textbf{80.7} & \textbf{63.6} & \textbf{47.2} \\
    \bottomrule
    \end{tabular}%
  }
  \label{table:main_exp_dft}
\end{table*}
\begin{table*}[t]
  \centering
  \caption{The main experiment results of membership inference on diffusion language models on instruction fine-tuning (IFT) tasks. We highlight the important values as \cref{table:main_exp_dft}.}
  \resizebox{\linewidth}{!}{
    \begin{tabular}{lccccccccc}
    \toprule
          & \multicolumn{3}{c}{MedQA (IFT)} & \multicolumn{3}{c}{Alpaca (IFT)} & \multicolumn{3}{c}{Tulu-3 (IFT)} \\
          \cmidrule(lr){2-4} \cmidrule(lr){5-7} \cmidrule(lr){8-10}
          & AUC   & TPR@10\%FPR & TPR@1\%FPR & AUC   & TPR@10\%FPR & TPR@1\%FPR & AUC   & TPR@10\%FPR & TPR@1\%FPR \\
    \midrule
    Loss  &   57.3    &   18.8    &  2.7     &   57.4    &  14.7     &  1.0     &   54.9   & 12.0      & 2.3 \\
    Min-K\% & 60.7 & 16.2  & 3.0   & 67.3 & 27.6  & 3.0   &   57.2    &   11.3    & 3.3 \\
    Min-K\%++& 49.7 & 8.7  & 2.3   & 59.6 & 21.2  & 1.7   &    55.1  &    13.0    & 2.7 \\
    Calibration & 55.6 & 17.1  & 1.0   & 61.7 & 20.3  & 5.0   &   60.1    &   \textbf{15.9}    & 2.6 \\
    $\text{SecMI}_\text{DLM}$ &  54.0     &  11.7     &   2.3    &    51.2   &   15.0    & 4.7      &  56.8     &  12.3     & 3.3 \\
    SAMA  & 65.1 & 26.0  & 3.2  & 69.1 & 31.0  & 11.2  &   59.7    &    12.1   &  3.7 \\
  \rowcolor[gray]{0.9}  \textbf{Q-Skew} & \textbf{72.8} & \textbf{31.6} & \textbf{6.7}   & \textbf{81.4} & \textbf{61.4} & \textbf{41.6} &   \textbf{67.2}   &  13.7     &  \textbf{4.3} \\
    \bottomrule
    \end{tabular}%
  }
  \label{table:main_exp_ift}
\end{table*}


\begin{algorithm}[t]
\caption{Quantile-weighted Skewness with Cyclic Sampling}
\label{alg:qskew_cyclic_mia}
\begin{algorithmic}
\small
\State \textbf{Input:} sample $\mathbf{x}=(w_1,\ldots,w_L)$, target model $p_{\theta_{\rm tar}}$, reference model $p_{\theta_{\rm ref}}$, cyclic rounds $R$, mask ratio $\rho$, bandwidth $h$, decision threshold $\tau_{\rm mia}$
\State \textbf{Output:} membership prediction $\hat{z}$ and inference score $a_{\mathbf{x}}$
\State $\mathcal{S}_{\mathbf{x}} \gets \varnothing$
\State $b \gets \max(1,\lfloor \rho L \rfloor)$
\For{$r=1,\ldots,R$}
    \State $\pi^{(r)} \gets \mathrm{Permutation}(\{1,\ldots,L\})$
    \State Partition $\pi^{(r)}$ into disjoint subsets $\{\mathcal{M}_{r,1},\ldots,\mathcal{M}_{r,m}\}$ with $|\mathcal{M}_{r,j}| \approx b$
    \For{$j=1,\ldots,m$}
        \State $\mathcal{M} \gets \mathcal{M}_{r,j}$
        \State Construct the unmasked context $\mathbf{x}_{\mathcal{M}^c}$
        \For{each $i \in \mathcal{M}$}
            \State $\ell^{\rm ref}_{i,\mathcal{M}} \gets -\log p_{\theta_{\rm ref}}(w_i \mid \mathbf{x}_{\mathcal{M}^c})$
            \State $\ell^{\rm tar}_{i,\mathcal{M}} \gets -\log p_{\theta_{\rm tar}}(w_i \mid \mathbf{x}_{\mathcal{M}^c})$
            \State $\Delta_{i,\mathcal{M}} \gets \ell^{\rm ref}_{i,\mathcal{M}} - \ell^{\rm tar}_{i,\mathcal{M}}$
            \State $\mathcal{S}_{\mathbf{x}} \gets \mathcal{S}_{\mathbf{x}} \cup \{\Delta_{i,\mathcal{M}}\}$
        \EndFor
    \EndFor
\EndFor
\State $\tilde{\mathcal{S}}_{\mathbf{x}} \gets \mathrm{Median}(\mathcal{S}_{\mathbf{x}})$
\State $\sigma_{\mathbf{x}} \gets \mathrm{Std}(\mathcal{S}_{\mathbf{x}})$
\For{each $s \in \mathcal{S}_{\mathbf{x}}$}
    \State $F(s) \gets \mathrm{ECDF}(s;\mathcal{S}_{\mathbf{x}})$
    \State Compute $W(s)$ using Eq.~\eqref{eq:qskew_weight}
\EndFor
\State $\bar{W} \gets \frac{1}{|\mathcal{S}_{\mathbf{x}}|}\sum_{s\in\mathcal{S}_{\mathbf{x}}} W(s)$
\State $a_{\mathbf{x}} \gets 
\frac{
\frac{1}{|\mathcal{S}_{\mathbf{x}}|}\sum_{s\in\mathcal{S}_{\mathbf{x}}} W(s)\cdot(s-\tilde{\mathcal{S}}_{\mathbf{x}})
}{
\sigma_{\mathbf{x}}\cdot \bar{W}
}$
\State $\hat{z} \gets \mathbbm{1}[a_{\mathbf{x}} > \tau_{\rm mia}]$
\State \Return $\hat{z}, a_{\mathbf{x}}$
\end{algorithmic}
\end{algorithm}


In practice, we 
(1) calculate the loss difference for every individual mask sampling instead of the mathematical expectation, increasing the set size for distribution estimation;
(2) design a cyclic sampling method to replace the random Monte Carlo sampling of $t$, eliminating the randomness in scenarios with a low number of samples, ensuring coverage uniformity;
(3) design a more robust metric based on \cref{eq:skewness},  to distinguish between the member set and the non-member set. We provide the entire MI progress in \cref{alg:qskew_cyclic_mia}.

\noindent\textbf{Token-level Memorization Score.}\
Let $\mathbf{x} = (w_1, \dots, w_L)$ be an input sequence. For any mask subset $\mathcal{M} \subset \{1, \dots, L\}$ sampled from the masking corruption process,
we define the instance loss difference $\Delta_{w_i, \mathcal{M}}$ for a masked token $w_i$ ($i \in \mathcal{M}$) as:
{
\begin{equation}
\label{eq:token_level_mem}
\begin{aligned}
    \Delta_{i, \mathcal{M}} := 
         & \left[ -\log p_{\theta_{\text{ref}}}(w_i \mid \mathbf{x}_{\mathcal{M}^c}) \right] \\
        & - \left[ -\log p_{\theta_{\text{tar}}}(w_i \mid \mathbf{x}_{\mathcal{M}^c}) \right],
\end{aligned}
\end{equation}
}\noindent
where $\mathbf{x}_{\mathcal{M}^c}$ denotes the context provided by unmasked tokens.

\noindent\textbf{Cyclic Sampling.}
To mitigate the stochastic variance inherent in Monte Carlo sampling and ensure uniform token coverage, we employ a Cyclic Sampling strategy. Instead of random replacement, we perform $R$ full-coverage round sampling. In the $r$-th rounds, we generate a random permutation of indices $\pi^{(r)} = \text{Permutation}(\{1, \dots, L\})$. This permutation is partitioned into a sequence of disjoint batches $\mathbf{B}^{(r)} = \{\mathcal{M}_{r,1}, \dots, \mathcal{M}_{r, m}\}$, such that $\bigcup_{j} \mathcal{M}_{r,j} = \{1, \dots, L\}$ and $\mathcal{M}_{r,j} \cap \mathcal{M}_{r,j'} = \emptyset$, where each batch has size $|\mathcal{M}| \approx \rho L$. The final memorization set aggregates the scores from all disjoint batches across all cycles:
\begin{equation}
    \mathcal{S}_{\mathbf{x}} = \bigcup_{r=1}^{R} \bigcup_{j=1}^{m} \left\{ \Delta_{i, \mathcal{M}_{r,j}} \mid i \in \mathcal{M}_{r,j} \right\}.
\end{equation}
\noindent\textbf{Quantile-weighted Skewness.}
In practice, extreme outliers can significantly inflate the mean of the entire set, making it difficult to distinguish between member and non-member samples. Therefore, we consider the median-based metric following Pearson's Second Skewness Coefficient:
\begin{equation} 
\frac{
\mu(\mathcal{S}_\mathbf{x}) - \tilde{\mathcal{S}_\mathbf{x}}
}{\sigma(\mathcal{S}_\mathbf{x})
},
\end{equation}
where $\tilde{\mathcal{S}_\mathbf{x}}$ denotes the median of $\mathcal{S}_\mathbf{x}$.
Since different quantiles exhibit varying impacts on membership inference accuracy, we further define the \textit{Quantile-weighted Skewness} (\shortname):
\begin{equation}
\begin{aligned}
    \mathcal{A'}(x, \theta_{\text{tar}}, \theta_{\text{ref}})  := & \\
   f_{\text{QSkew}}(\mathbf{x})  = & \frac{\frac{1}{|\mathcal{S}_\mathbf{x}|} \sum_{s \in \mathcal{S}_\mathbf{x}} W(s) \cdot (s - \tilde{\mathcal{S}_\mathbf{x}})}{\sigma(\mathcal{S}_\mathbf{x}) \cdot \bar{W}},
\end{aligned}
\end{equation}
where $\bar{W}$ is a normalization factor: 
$$\bar{W} = \frac{1}{|\mathcal{S}_\mathbf{x}|} \sum_{s \in \mathcal{S}_\mathbf{x}} W(s).$$
To mitigate the influence of outliers and focus on informative regions, we formulate the weight function $W(s)$ using the empirical cumulative distribution function (ECDF):
\begin{equation}
\label{eq:qskew_weight}
\scalebox{0.9}{$
    W(s) = \exp\left(-\frac{(F(s) - 0.15)^2}{2h^2}\right) + \exp\left(-\frac{(F(s) - 0.85)^2}{2h^2}\right),
$}
\end{equation}
where $F(s) \in [0, 1]$ represents the quantile position of $s$ within $\mathcal{S}_\mathbf{x}$, and $h$ is a bandwidth hyperparameter (set to 0.1 in experiment). By anchoring weights at the 15th and 85th percentiles, this metric captures distributional asymmetry while remaining robust to long-tailed distributions.

\section{Experiments}
\label{sec:evaluation}

\topic{Models and Datasets.}
To validate the generalizability of our method, we broadly consider mainstream open-source DLMs, including LLaDA-8B-Base~\cite{llada_8b_base}, LLaDA-8B-Instruct~\cite{llada_8b_instruct}, Dream-Base-7B~\cite{Dream_v0_Base_7B}, and Dream-Instruct-7B~\cite{Dream_v0_Instruct_7B}, under two fine-tuning settings: domain fine-tuning (DFT) and instruction fine-tuning (IFT), as described in \cref{subsec:threat}. 
We also broadly consider six datasets from diverse domains. 
Specifically, for (1) domain finetuning (DFT), we finetune the Base DLMs with ArXiv \cite{realtimedata_arxiv_alltime}, WikiText \cite{realtimedata_wikitext_alltime}, and XSUM \cite{narayan2018don} datasets.
And for (2) instruction finetuning (IFT), we finetune the Instruct DLMs with MedQA \cite{li2023chatdoctor}, Alpaca \cite{alpaca}, and Tulu-3 \cite{lambert2024tulu} datasets.

Note that we do not consider MIMIR \citep{mimir_dataset} used in \cite{chen2026membership}, because the \textit{training (member) and test (non-member) sets in MIMIR already exhibit a prior discrepancy} within the original model, making it unsuitable for evaluating fine-tuning scenarios. 
We detail the validation in \cref{sec:hallu}.

\topic{Baselines.}
We consider a broad range of membership inference baselines from the following sources: 
\ding{182} adaptations of methods targeting traditional auto-regressive language models: Loss~\cite{yeom2018privacy}, Min-K\%~\cite{shi2023detecting}, Min-K\%++~\cite{zhang2024min}, Calibration~\cite{watson2021importance}; 
\ding{183} adaptations of methods targeting (vision) diffusion models\footnote{Note that some baselines~\cite{kong2024an,zhai2024membership} are not suitable for adaptation to DLMs that are not included. For instance, the theoretical derivations of PIA~\cite{kong2024an} are designed for continuous diffusion and differ fundamentally from DLMs. 
}: SecMI~\cite{duan2023diffusion};
and \ding{184} one concurrent work: SAMA~\cite{chen2026membership}. 
We detail the rationale for baseline selection and the adaptation design from their original target models to DLMs in Appendix \ref{appd:baselines}.

\topic{Metrics.}
We consider the widely used evaluation metrics, including the area under the receiver operating characteristic curve (AUC), the True Positive Rate when the False Positive Rate of 10\% and 1\% (TPR@10\%FPR, TPR@1\%FPR), following previous works~\cite{carlini2022membership,duan2023diffusion,fu2024membership,chen2026membership}.

\topic{Implementation.}
We conduct three independent runs and report the average of each metric to reduce random error.
We uniformly train for 6 epochs with a batch size of 16. We also consider other training settings in \cref{sec:epochs}.
For all methods, we apply 16 denoising steps per sample to ensure consistent computational complexity.
When calculating the token-level memorization score, we set $\rho=0.35$ for calculating \cref{eq:token_level_mem}, as this noise level maximizes the distinction between the member and non-member sets.  We also conduct ablation studies to show the effectiveness of each component of our method in \cref{sec:ablation}.

\subsection{Main Results}
\label{subsec:effectiveness-mi}
We evaluate our method (denoted as \shortname) and the baselines on LLaDA-8B-Base~\cite{llada_8b_base} and LLaDA-8B-Instruct~\cite{llada_8b_instruct} across various datasets and report the results in \cref{table:main_exp_dft} and \cref{table:main_exp_ift}.
Experimental results show that our method consistently outperforms the baselines across different datasets and DLMs, and training settings, including domain fine-tuning and instruction fine-tuning.

\begin{table}[t]
  \centering
  \caption{The MI performance on XSUM with Dream-Base-7B for domain fine-tuning and on Tulu-3 with Dream-7B-Instruct for instruction fine-tuning. We highlight the important values as \cref{table:main_exp_dft}.}
  \resizebox{\linewidth}{!}{
    \begin{tabular}{lcccccc}
    \toprule
          & \multicolumn{3}{c}{XSUM (DFT)} & \multicolumn{3}{c}{Tulu-3 (IFT)} \\
          \cmidrule(lr){2-4} \cmidrule(lr){5-7}
          & AUC & TPR@10\% & TPR@1\% & AUC & TPR@10\% & TPR@1\% \\
    \midrule
    Loss & 55.0 & 10.4 & 2.2 & 52.7 & 10.2 & 2.1 \\
    MinK & 57.1 & 12.0 & 2.7 & 53.6 & 10.7 & 2.5 \\
    MinK++ & 56.4 & 11.2 & 3.0 & 51.8 & 8.9 & 2.0 \\
    Calibration & 53.6 & 9.5 & 2.1 & 55.1 & 12.0 & 2.7 \\
    $\text{SecMI}_\text{DLM}$ & 57.5 & 12.8 & 3.5 & 53.4 & 9.8 & 2.3 \\
    SAMA & 60.7 & 14.5 & \textbf{4.4} & 54.8 & 11.1 & 2.8 \\
    \rowcolor[gray]{0.9} \textbf{Q-Skew} & \textbf{63.4} & \textbf{15.0} & 4.2 & \textbf{59.1} & \textbf{12.7} & \textbf{3.1} \\
    \bottomrule
    \end{tabular}%
  }
  \label{table:main_dream}
\end{table}

Besides the LLaDA series, we conduct additional evaluations using the Dream-7B series models, including Dream-Base-7B~\cite{Dream_v0_Base_7B} and Dream-Instruct-7B~\cite{Dream_v0_Instruct_7B} on XSUM~\cite{narayan2018don} for domain fine-tuning and Tulu-3~\cite{lambert2024tulu} for instruction fine-tuning, respectively.
The results show that our method still achieves the best overall performance among all baselines.
Note that, compared with the LLaDA series models, the Dream series models exhibit lower vulnerability in membership inference. A similar observation is also reported in concurrent work~\cite{chen2026membership}.

\subsection{Performance on Various Training Epochs}
\label{sec:epochs}
\begin{figure}[t]
    \centering
    \includegraphics[width=1\linewidth]{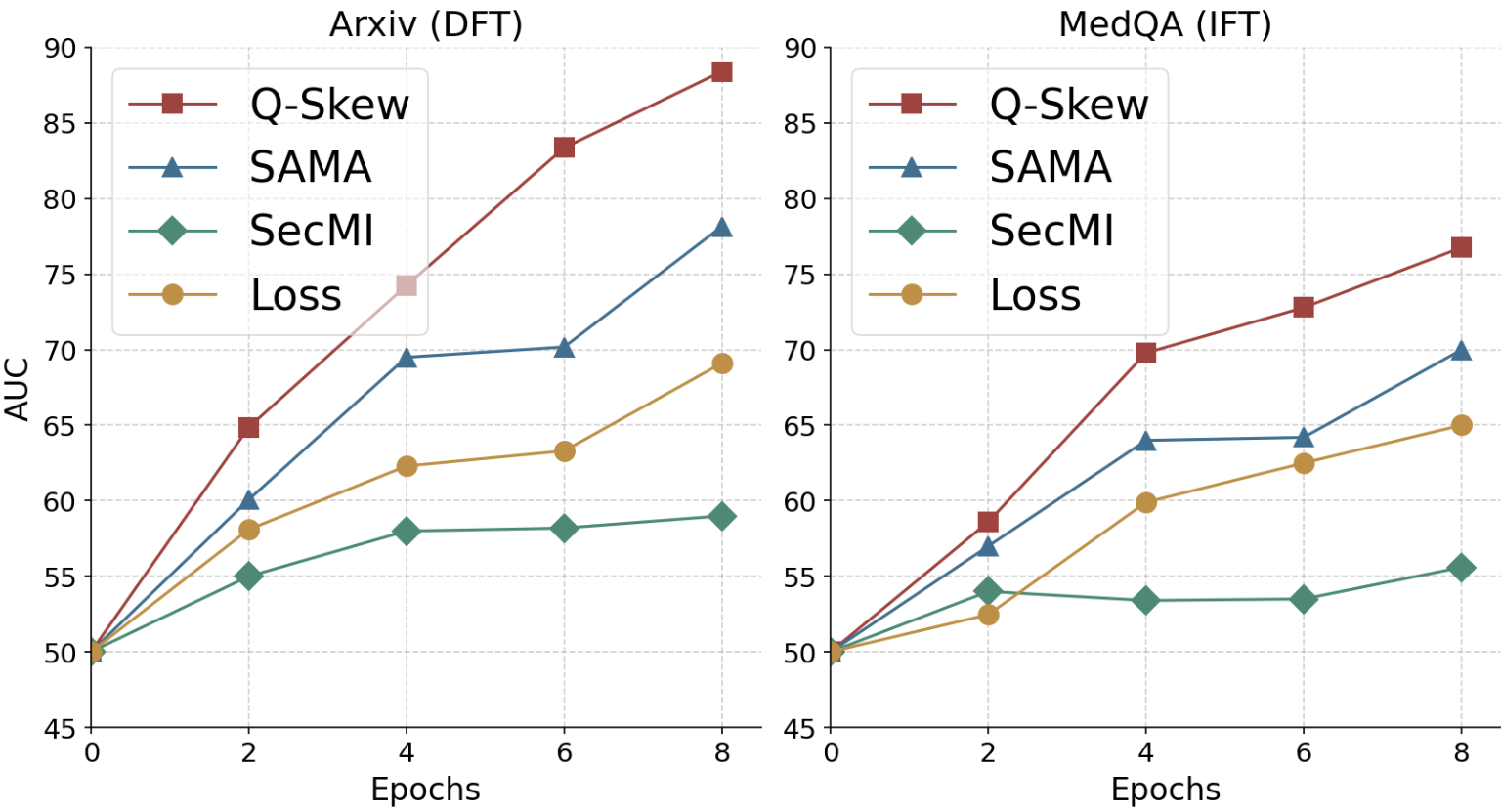}
    \caption{Effectiveness under different training epochs.}
    \label{fig:auc_epoch}
\end{figure}
As fine-tuning progresses, the model's memorization of the training data gradually increases. In practical scenarios, the number of fine-tuning steps by the model trainer varies. Membership inference methods that more effectively reveal membership information across different steps are considered superior~\cite{zhai2024membership,lian2025unveiling}. 
We evaluate the effectiveness of the membership inference method across different training epochs in \cref{fig:auc_epoch}. Experimental results demonstrate that our method consistently outperforms the baselines across different training steps.

\subsection{Weaker Assumption of Reference Models}
\label{sec:weaker}

\begin{table}[t]
  \centering
  \caption{Performance of membership inference methods under a weaker assumption: using misaligned reference models}
  \label{tab:mis_align}

\setlength{\tabcolsep}{10pt}
  
    \resizebox{\linewidth}{!}{
  \begin{tabular}{lccc}
    \toprule
    & \multicolumn{3}{c}{Base DLM} \\
    \cmidrule(l){2-4}
    & AUC & TPR@10 & TPR@1 \\
    \midrule
    Golden Reference & 83.4 & 66.1 & 50.6 \\
    \midrule
    Instruction Reference & \textbf{78.4} (\textcolor{blue}{-5.0})  & 55.5 (\textcolor{blue}{-10.7}) & \textbf{45.5} (\textcolor{blue}{-5.1}) \\
    Para-modification & 78.4 (\textcolor{blue}{-5.0}) & \textbf{60.0} (\textcolor{blue}{-6.2}) & 45.0 (\textcolor{blue}{-5.6}) \\
    SAMA (w/ Ref)  & 70.2  & 34.4  & 17.6  \\
    \bottomrule
  \end{tabular}
  }
  
  \vspace{1em} 
    
\resizebox{\linewidth}{!}{
  \begin{tabular}{lccc}
    \toprule
    & \multicolumn{3}{c}{Instruction DLM} \\
    \cmidrule(l){2-4}
    & AUC & TPR@10 & TPR@1 \\
    \midrule
    Golden Reference &  72.8 & 31.6 & 6.7 \\
    \midrule
    Base Reference & \textbf{71.5} (\textcolor{blue}{-1.3}) & \textbf{32.0} (\textcolor{red}{+0.4}) & \textbf{6.0} (\textcolor{blue}{-0.7}) \\
    Para-modification & 70.0 (\textcolor{blue}{-2.8}) & 25.0 (\textcolor{blue}{-6.6}) & 2.5 (\textcolor{blue}{-4.2})  \\
    SAMA (w/ Ref)  & 65.1 & 26.0 & 3.2  \\
    \bottomrule
  \end{tabular}

  }
  
\end{table}

Although it is reasonable to assume that the MI adversary can access an opensource unfinetuned reference model in the finetuning setting, we additionally consider the effectiveness of our method when the used reference model is misaligned.
We consider the following three alternative strategies when the unfinetuned model corresponding to the target model is unavailable: 
(1) Use an unfinetuned base model as a substitute for the unfinetuned instruction model.
(2) Use an unfinetuned instruction model as a substitute for the unfinetuned base model.
(3) Use a model that has been fine-tuned on other perturbed data as a substitute for the unfinetuned model.

We test the situation of domain fine-tuning on the ArXiv dataset and instruction fine-tuning on the MedQA dataset. In \cref{tab:mis_align}, while the membership inference performance of our method exhibits a slight decline when the reference model is misaligned, this decrease is minor, and our approach still outperforms the best baseline. We provide the justification of the assumption that the adversary can access a reference model, although it may not be fully aligned with the target model, in Appendix~\ref{appd:justify}.

\subsection{Ablation Studies}
\label{sec:ablation}
\begin{table}[t]
  \centering
  \caption{Performance of ablation studies on the domain finetuning task.}
    \resizebox{\linewidth}{!}{

  \begin{tabular}{lccc}
    \toprule
     & AUC & TPR@10\%FPR & TPR@1\%FPR \\
    \midrule
    Q-Skew & 80.3 & 57.7 & 39.7 \\
    \midrule
    w/o Cyclic Sampling & 78.5 & 53.0 & 37.0 \\
    w/o Skewness Information & 77.6 & 54.3 & 34.0 \\
    w/o quantile weighting & 78.2 & 55.0 & 33.4 \\
    \bottomrule
  \end{tabular}
  }
  \label{tab:ablation}
\end{table}
We conduct an ablation study on the domain fine-tuning task using WikiText~\cite{realtimedata_wikitext_alltime}.
We consider the following settings:
(1) without the \textbf{cyclic sampling} strategy.
(2) replacing the \textbf{skewness} calculation with the mean value.
(3) using the skewness metric without \textbf{quantile weighting}.
The experimental results in \cref{tab:ablation} show that each component in the method is effective.

\section{
Increase in Privacy Risks Beyond MIAs}
\label{sec:pii-reconstruction}

Previously, we showed that our skewness metric 
improves membership inference attacks.
We now consider a complementary privacy risk 
in generative language models: 
\emph{data reconstruction attacks}%
~\cite{carlini2018secret, carlini2023quantifying, hayes2024measuring, biderman2023emergent, lee2022deduplicating, ippolito2023preventing, wang2026silent}.
Specifically, we investigate whether \shortname, 
designed to capture the distributional difference 
between seen and unseen tokens---%
can be 
effective in 
extracting personally identifiable information (PII).

\subsection{Threat Model}
\label{subsec:threat-model}

We consider a PII reconstruction adversary%
~\cite{lukas2023analyzing, meng2025rr, nakka2024pii, kim2023propile}
who is given a training record with PII redacted 
and aims to recover the redacted tokens 
via queries to the target model.
For example, given the PII-redacted record
\texttt{``Insurance 
Member ID: [REDACTED]''},
the objective is to reconstruct the true value
\texttt{``XJ45-7831-92''}.
The adversary can generate \emph{candidate PII values} 
and \emph{rank} them. 

\subsection{PII Reconstruction for DLMs}
\label{subsec:pii-recon-dlms}

PII reconstruction attacks have been widely studied for autoregressive language models,
where attackers query the model and identify memorized sequences
via high conditional probabilities.
%
However, these approaches do \emph{not} directly apply to DLMs.
(1) DLMs generate tokens through iterative denoising 
rather than left-to-right prediction, and thus 
do not expose conditional probabilities of the form $P(x_t \mid x_{<t})$.
(2) Prior attacks rely on directly prompting 
models to reproduce memorized sequences,
whereas DLMs reconstruct tokens 
within partially noisy contexts, 
making such generation-based extraction less reliable.
\smallskip

\topic{Candidate-based PII Extraction.}
To adapt existing approaches, we 
formulate 
PII extraction as a candidate ranking problem.
Given a context containing a masked PII record 
and candidate PII items $\{z_1, z_2, \dots, z_k\}$, 
the attacker aims to identify the most likely memorized candidate.
Candidates are generated via random remasking 
across $N=64$ independent denoising process.
Each candidate $z_i$ is inserted into the masked position,
and its plausibility is estimated by computing the cross-entropy loss $l_{CE}$
under stochastic masking.
This likelihood-based ranking serves as a baseline.
\smallskip

\topic{Incorporating Our Skewness Score.}
Previous likelihood-based ranking provides a strong baseline 
for PII extraction by measuring the reconstruction loss of candidate completions.
However, this approach primarily reflects the overall sequence likelihood
under the model's \emph{language prior},
which can cause different candidates to appear similarly plausible.
For example, PII records such as \texttt{``john.doe@gmail.com''}
and \texttt{``alice.bob@gmail.com''} may both receive similar $l_{CE}$ values
because common tokens like \texttt{``gmail.com''} 
are strongly supported by the language prior.
As a result, the ranking may be dominated by 
generic token probabilities rather than signals of memorization.
Our skewness score (token-level asymmetry) in contrast,
captures localized reconstruction behavior and 
can highlight asymmetries that arise from memorized sequences.
This makes it a \emph{complementary} signal for candidate ranking.
To combine the strengths of both signals, we compute a weighted score:
\begin{align*}
    \text{Score}(z_i) = & -(1-\alpha) \cdot Z(l_{CE}(z_i)) \\
                 &+ (\alpha) \cdot Z(\text{Skew}_{\text{Pearson2}}(z_i))
\end{align*}
where $Z$ denotes z-score normalization.
In practice, an adversary may have access to 
a few training records and use them to estimate an 
effective $\alpha$. 
We show the impact of $\alpha$ in Appendix~\ref{appendix:recon-abl}.
To better capture localized memorization effects, 
we compute skewness over a contextual span of length $X$
centered on the PII record.
For example, when $X\!=\!20$, the span includes 
8 preceding 
tokens, the 5 PII tokens, and 7 following tokens.
For each candidate sequence, we generate masked variants 
by randomly masking 35\% of tokens. 
%
We repeat this process for 4--16 batches.
Token-level loss differences are computed across these masked variants 
and aggregated to estimate the final score.

\subsection{Evaluation}
\label{subsec:eval}

We evaluate our attacks on two types of PII:
\emph{phone} numbers and \emph{email} addresses.
%
We use the TREC dataset~\cite{wu2006an},
which contains 174,299 records, split into 
78k, 78k, and 17k examples for training, validation, and testing, respectively.
Consistent with our membership inference setup,
we fine-tune the LLaDA-8B-Base model on the training set.
We identify PII instances 
applying regular expressions that match email addresses and phone numbers,
and randomly sample 50 unique email addresses and 50 unique phone numbers as targets.
The attacker generates candidate PII
by inserting 20 mask tokens at the target location within a record and 
running 64 independent diffusion processes to fill them,
then selects the most likely candidate 
using the attack signals defined in \S\ref{subsec:pii-recon-dlms}.

\topic{Metrics and Baselines.} For our baseline, we use cross entropy loss as the candidate-scoring mechanism. 
We compute the cross-entropy loss over the candidate PII tokens for each candidate in a set, and then rank them. 
Our attack success metric, top-1 accuracy, measures the proportion of targets in which the ground-truth PII achieves the highest rank.
Following prior work, we report top-1 accuracy as a value between 0 and 100.
In the case of cross-entropy loss we expect the ground-truth PII to have a lower loss compared to the candidates.
In contrast, we expect the ground-truth PII to have a higher skewness relative to other candidates.
\smallskip

\begin{table}[t]
  \centering
  \caption{%
    ASR (\%) of our attack methods on LLaDA-8B-Base
    for phone numbers and email addresses.
    Higher is better.
    The best results on \texttt{Combine-0.1} are in bold.
  }
  \vspace{-0.6em}
  \resizebox{\linewidth}{!}{
    \begin{tabular}{llcc}
    \toprule
    \multirow{2}{*}{\textbf{Our Attacks}} & \multirow{2}{*}{\textbf{Signal}} & \multicolumn{2}{c}{\textbf{ASR} (\%)} \\ \cmidrule(lr){3-4}
     & & \textbf{Phone} & \textbf{Email} \\ \midrule \midrule
    \texttt{Ranking} & CE & 30 [15/50] & 20 [10/50] \\
    \texttt{Mask only PII} & \shortname & 22 [11/50] & 16 [\hspace{0.5em}8/50] \\
    \texttt{Mask only context} & \shortname & 10 [\hspace{0.5em}5/50] & \hspace{0.5em}8 [\hspace{0.5em}4/50] \\
    \texttt{Combine-0.1} & Both & \textbf{34} [17/50] & \textbf{24} [12/50] \\
    \bottomrule
    \end{tabular}%
  }
  \label{tab:pii_recon}
  \vspace{-1.0em}
\end{table}

\topic{Results.}
Table~\ref{tab:pii_recon} 
shows our results.
We compare four 
attacks:
cross-entropy (CE)-based ranking (\texttt{Ranking}),
two skewness-based variants that mask either
only PII tokens (\texttt{Mask only PII})
or only context tokens (\texttt{Mask only context}),
and a combined method (\texttt{Combine-0.1}) that integrates CE and skewness signals.
The skewness-based method names reflect
whether the randomly masked tokens are applied to PII or context tokens.
The value 0.1 denotes the $\alpha$ value used to combine the two
(see Appendix~\ref{appendix:recon-abl} for details 
on how we select $\alpha$).

We show that CE-based ranking serves as a strong baseline, 
achieving 30\% and 20\% ASR on phone numbers and emails. 
Skewness-based methods alone performs worse 
(10--22\% ASR on phone numbers and 8--16\% on emails), 
reflecting that PII reconstruction is primarily a ranking task 
rather than pure inference.
However, combining CE with skewness produces the best results, 
improving ASR to 34\% for phone numbers and 24\% for emails.
This suggests that the skewness helps mitigate 
the influence of strong language priors acquired 
during pre-training, which can bias CE-based ranking.
Additional qualitative analysis supporting this claim is in Appendix~\ref{appendix:pii-qual-analysis},
and more ablation studies are in 
Appendix~\ref{appendix:recon-abl} due to space limits.

\section{Conclusion}
\label{sec:conclusion}

In this paper, we first define the \textit{token-level memorization asymmetry}. Based on this, we theoretically derive an MI method for fine-tuned DLMs that outperforms existing baselines across datasets and setups. Furthermore, we demonstrate that our method facilitates PII reconstruction.
Our work indicates that, due to differences in the training process, DLMs exhibit distinct characteristics from autoregressive (AR) LLMs,  encouraging the community to focus on the unique privacy risks of DLMs.

\section*{Limitations}

While our method advances the privacy-violating attacks such as membership inference and PII extraction in diffusion language models, several limitations remain.  First, because publicly available open source diffusion language models that are suitable for training are still limited, we mainly evaluate our method on current mainstream models.
Moreover, diffusion language models are developing rapidly.
More diverse diffusion architectures and training mechanisms may emerge in the future, and the generalization ability of our method to such future models remains unknown.

\section*{Impact Statement}

This work advances our understanding of privacy risks in DLMs,
an emerging class of generative models.
The proposed attack, 
\shortname
, is intended as an auditing tool 
to assess and mitigate privacy risks, 
rather than to facilitate misuse.
We believe that exposing this vulnerability is a necessary step 
toward developing effective privacy defenses 
and informing the responsible deployment of DLMs.
More broadly, this work underscores the importance of 
evaluating privacy guarantees in emerging model architectures 
and encourages the development of privacy-preserving training 
and inference techniques for diffusion-based models.


\bibliography{bibs/thiswork}

@inproceedings{yudifferentially,
  title={Differentially Private Fine-tuning of Language Models},
  author={Yu, Da and Naik, Saurabh and Backurs, Arturs and Gopi, Sivakanth and Inan, Huseyin A and Kamath, Gautam and Kulkarni, Janardhan and Lee, Yin Tat and Manoel, Andre and Wutschitz, Lukas and others},
  booktitle={International Conference on Learning Representations}
}

@article{fu2024membership,
  title={Membership inference attacks against fine-tuned large language models via self-prompt calibration},
  author={Fu, Wenjie and Wang, Huandong and Gao, Chen and Liu, Guanghua and Li, Yong and Jiang, Tao},
  journal={Advances in Neural Information Processing Systems},
  volume={37},
  pages={134981--135010},
  year={2024}
}

@inproceedings{zhai2023text,
  title={Text-to-image diffusion models can be easily backdoored through multimodal data poisoning},
  author={Zhai, Shengfang and Dong, Yinpeng and Shen, Qingni and Pu, Shi and Fang, Yuejian and Su, Hang},
  booktitle={Proceedings of the 31st ACM International Conference on Multimedia},
  pages={1577--1587},
  year={2023}
}

@article{zhai2024membership,
  title={Membership inference on text-to-image diffusion models via conditional likelihood discrepancy},
  author={Zhai, Shengfang and Chen, Huanran and Dong, Yinpeng and Li, Jiajun and Shen, Qingni and Gao, Yansong and Su, Hang and Liu, Yang},
  journal={Advances in Neural Information Processing Systems},
  volume={37},
  pages={74122--74146},
  year={2024}
}

@inproceedings{shokri2017membership,
  title={Membership inference attacks against machine learning models},
  author={Shokri, Reza and Stronati, Marco and Song, Congzheng and Shmatikov, Vitaly},
  booktitle={2017 IEEE symposium on security and privacy (SP)},
  pages={3--18},
  year={2017},
  organization={IEEE}
}

@article{dealcala2024my,
  title={Is my data in your ai model? membership inference test with application to face images},
  author={DeAlcala, Daniel and Morales, Aythami and Fierrez, Julian and Mancera, Gonzalo and Tolosana, Ruben and Ortega-Garcia, Javier},
  journal={arXiv preprint arXiv:2402.09225},
  year={2024}
}

@article{ho2020denoising,
  title={Denoising diffusion probabilistic models},
  author={Ho, Jonathan and Jain, Ajay and Abbeel, Pieter},
  journal={Advances in neural information processing systems},
  volume={33},
  pages={6840--6851},
  year={2020}
}

@inproceedings{
kong2024an,
title={An Efficient Membership Inference Attack for the Diffusion Model by Proximal Initialization},
author={Fei Kong and Jinhao Duan and RuiPeng Ma and Heng Tao Shen and Xiaoshuang Shi and Xiaofeng Zhu and Kaidi Xu},
booktitle={The Twelfth International Conference on Learning Representations},
year={2024},
url={https://openreview.net/forum?id=rpH9FcCEV6}
}

@misc{fu2024probabilisticfluctuationbasedmembership,
      title={A Probabilistic Fluctuation based Membership Inference Attack for Diffusion Models}, 
      author={Wenjie Fu and Huandong Wang and Liyuan Zhang and Chen Gao and Yong Li and Tao Jiang},
      year={2024},
      eprint={2308.12143},
      archivePrefix={arXiv},
      primaryClass={cs.LG},
      url={https://arxiv.org/abs/2308.12143}, 
}

@inproceedings{
watson2022on,
title={On the Importance of Difficulty Calibration in Membership Inference Attacks},
author={Lauren Watson and Chuan Guo and Graham Cormode and Alexandre Sablayrolles},
booktitle={International Conference on Learning Representations},
year={2022},
url={https://openreview.net/forum?id=3eIrli0TwQ}
}

@inproceedings{
chen2026membership,
title={Membership Inference Attacks Against Fine-tuned Diffusion Language Models},
author={Yuetian Chen and Kaiyuan Zhang and Yuntao Du and Edoardo Stoppa and Charles Fleming and Ashish Kundu and Bruno Ribeiro and Ninghui Li},
booktitle={The Fourteenth International Conference on Learning Representations},
year={2026},
url={https://openreview.net/forum?id=oWKJursYpW}
}

@inproceedings{carlini2022membership,
  title={Membership inference attacks from first principles},
  author={Carlini, Nicholas and Chien, Steve and Nasr, Milad and Song, Shuang and Terzis, Andreas and Tramer, Florian},
  booktitle={2022 IEEE symposium on security and privacy (SP)},
  pages={1897--1914},
  year={2022},
  organization={IEEE}
}

@article{lian2025unveiling,
  title={Unveiling impact of frequency components on membership inference attacks for diffusion models},
  author={Lian, Puwei and Cai, Yujun and Li, Songze and Bao, Bingkun},
  journal={arXiv preprint arXiv:2505.20955},
  year={2025}
}

@article{lukas2023analyzing,
  title={Analyzing Leakage of Personally Identifiable Information in Language Models},
  author={Nils Lukas and A. Salem and Robert Sim and Shruti Tople and Lukas Wutschitz and Santiago Zanella-B'eguelin},
  journal={2023 IEEE Symposium on Security and Privacy (SP)},
  year={2023},
  pages={346-363},
  url={https://api.semanticscholar.org/CorpusID:256459554}
}

@inproceedings{nakka2024pii,
    title = "{PII}-Compass: Guiding {LLM} training data extraction prompts towards the target {PII} via grounding",
    author = "Nakka, Krishna Kanth  and
      Frikha, Ahmed  and
      Mendes, Ricardo  and
      Jiang, Xue  and
      Zhou, Xuebing",
    editor = "Habernal, Ivan  and
      Ghanavati, Sepideh  and
      Ravichander, Abhilasha  and
      Jain, Vijayanta  and
      Thaine, Patricia  and
      Igamberdiev, Timour  and
      Mireshghallah, Niloofar  and
      Feyisetan, Oluwaseyi",
    booktitle = "Proceedings of the Fifth Workshop on Privacy in Natural Language Processing",
    month = aug,
    year = "2024",
    address = "Bangkok, Thailand",
    publisher = "Association for Computational Linguistics",
    url = "https://aclanthology.org/2024.privatenlp-1.7/",
    pages = "63--73"
}

@article{nie2025large,
  title={Large language diffusion models},
  author={Nie, Shen and Zhu, Fengqi and You, Zebin and Zhang, Xiaolu and Ou, Jingyang and Hu, Jun and Zhou, Jun and Lin, Yankai and Wen, Ji-Rong and Li, Chongxuan},
  journal={arXiv preprint arXiv:2502.09992},
  year={2025}
}

@article{bie2025llada2,
  title={Llada2. 0: Scaling up diffusion language models to 100b},
  author={Bie, Tiwei and Cao, Maosong and Chen, Kun and Du, Lun and Gong, Mingliang and Gong, Zhuochen and Gu, Yanmei and Hu, Jiaqi and Huang, Zenan and Lan, Zhenzhong and others},
  journal={arXiv preprint arXiv:2512.15745},
  year={2025}
}

@article{ye2025dream,
  title={Dream 7b: Diffusion large language models},
  author={Ye, Jiacheng and Xie, Zhihui and Zheng, Lin and Gao, Jiahui and Wu, Zirui and Jiang, Xin and Li, Zhenguo and Kong, Lingpeng},
  journal={arXiv preprint arXiv:2508.15487},
  year={2025}
}

@article{song2025seed,
  title={Seed diffusion: A large-scale diffusion language model with high-speed inference},
  author={Song, Yuxuan and Zhang, Zheng and Luo, Cheng and Gao, Pengyang and Xia, Fan and Luo, Hao and Li, Zheng and Yang, Yuehang and Yu, Hongli and Qu, Xingwei and others},
  journal={arXiv preprint arXiv:2508.02193},
  year={2025}
}

@misc{gemini_diffusion_2025,
  author       = {{Google DeepMind}},
  title        = {Gemini Diffusion},
  year         = {2025},
  howpublished = {\url{https://deepmind.google/models/gemini-diffusion/}},
  note         = {Accessed: 2026-03-16}
}

@misc{inception_mercury_chat_2025,
  author       = {{Inception Labs}},
  title        = {Introducing Mercury, Our General Chat Diffusion Large Language Model},
  year         = {2025},
  howpublished = {\url{https://www.inceptionlabs.ai/blog/introducing-mercury-our-general-chat-model}},
  note         = {Accessed: 2026-03-16}
}

@inproceedings{de2025accelerated,
  title={Accelerated Diffusion Models via Speculative Sampling},
  author={De Bortoli, Valentin and Galashov, Alexandre and Gretton, Arthur and Doucet, Arnaud},
  booktitle={International Conference on Machine Learning},
  pages={12590--12631},
  year={2025},
  organization={PMLR}
}

@article{li2022diffusion,
  title={Diffusion-lm improves controllable text generation},
  author={Li, Xiang and Thickstun, John and Gulrajani, Ishaan and Liang, Percy S and Hashimoto, Tatsunori B},
  journal={Advances in neural information processing systems},
  volume={35},
  pages={4328--4343},
  year={2022}
}

@inproceedings{yeom2018privacy,
  title={Privacy risk in machine learning: Analyzing the connection to overfitting},
  author={Yeom, Samuel and Giacomelli, Irene and Fredrikson, Matt and Jha, Somesh},
  booktitle={2018 IEEE 31st computer security foundations symposium (CSF)},
  pages={268--282},
  year={2018},
  organization={IEEE}
}

@article{shi2023detecting,
  title={Detecting pretraining data from large language models},
  author={Shi, Weijia and Ajith, Anirudh and Xia, Mengzhou and Huang, Yangsibo and Liu, Daogao and Blevins, Terra and Chen, Danqi and Zettlemoyer, Luke},
  journal={arXiv preprint arXiv:2310.16789},
  year={2023}
}

@article{zhang2024min,
  title={Min-k\%++: Improved baseline for detecting pre-training data from large language models},
  author={Zhang, Jingyang and Sun, Jingwei and Yeats, Eric and Ouyang, Yang and Kuo, Martin and Zhang, Jianyi and Yang, Hao Frank and Li, Hai},
  journal={arXiv preprint arXiv:2404.02936},
  year={2024}
}

@article{watson2021importance,
  title={On the importance of difficulty calibration in membership inference attacks},
  author={Watson, Lauren and Guo, Chuan and Cormode, Graham and Sablayrolles, Alex},
  journal={arXiv preprint arXiv:2111.08440},
  year={2021}
}

@inproceedings{duan2023diffusion,
  title={Are diffusion models vulnerable to membership inference attacks?},
  author={Duan, Jinhao and Kong, Fei and Wang, Shiqi and Shi, Xiaoshuang and Xu, Kaidi},
  booktitle={International Conference on Machine Learning},
  pages={8717--8730},
  year={2023},
  organization={PMLR}
}

@article{ni2025diffusion,
  title={Diffusion language models are super data learners},
  author={Ni, Jinjie and Liu, Qian and Dou, Longxu and Du, Chao and Wang, Zili and Yan, Hang and Pang, Tianyu and Shieh, Michael Qizhe},
  journal={arXiv preprint arXiv:2511.03276},
  year={2025}
}

@inproceedings{meng2025rr,
  title={Rr: Unveiling llm training privacy through recollection and ranking},
  author={Meng, Wenlong and Zhenyuan, Guo and Wu, Lenan and Gong, Chen and Liu, Wenyan and Li, Weixian and Wei, Chengkun and Chen, Wenzhi},
  booktitle={Findings of the Association for Computational Linguistics: ACL 2025},
  pages={17383--17397},
  year={2025}
}

@inproceedings{huang2025df,
  title={Df-mia: A distribution-free membership inference attack on fine-tuned large language models},
  author={Huang, Zhiheng and Liu, Yannan and He, Daojing and Li, Yu},
  booktitle={Proceedings of the AAAI Conference on Artificial Intelligence},
  volume={39},
  number={1},
  pages={343--351},
  year={2025}
}

@inproceedings{narayan2018don,
  title={Don’t give me the details, just the summary! topic-aware convolutional neural networks for extreme summarization},
  author={Narayan, Shashi and Cohen, Shay B and Lapata, Mirella},
  booktitle={Proceedings of the 2018 conference on empirical methods in natural language processing},
  pages={1797--1807},
  year={2018}
}

@inproceedings{carlini2018secret,
  title={The Secret Sharer: Evaluating and Testing Unintended Memorization in Neural Networks},
  author={Nicholas Carlini and Chang Liu and {\'U}lfar Erlingsson and Jernej Kos and Dawn Xiaodong Song},
  booktitle={USENIX Security Symposium},
  year={2018},
  url={https://api.semanticscholar.org/CorpusID:170076423}
}

@inproceedings{wu2006an,
  title={An Exploratory Study of the W3C Mailing List Test Collection for Retrieval of Emails with Pro/Con Argument},
  author={Yejun Wu and Douglas W. Oard and Ian Soboroff},
  booktitle={International Conference on Email and Anti-Spam},
  year={2006},
  url={https://api.semanticscholar.org/CorpusID:10214775}
}

@inproceedings{
carlini2023quantifying,
title={Quantifying Memorization Across Neural Language Models},
author={Nicholas Carlini and Daphne Ippolito and Matthew Jagielski and Katherine Lee and Florian Tramer and Chiyuan Zhang},
booktitle={The Eleventh International Conference on Learning Representations },
year={2023},
url={https://openreview.net/forum?id=TatRHT_1cK}
}

@inproceedings{
kim2023propile,
title={Pro{PILE}: Probing Privacy Leakage in Large Language Models},
author={Siwon Kim and Sangdoo Yun and Hwaran Lee and Martin Gubri and Sungroh Yoon and Seong Joon Oh},
booktitle={Thirty-seventh Conference on Neural Information Processing Systems},
year={2023},
url={https://openreview.net/forum?id=QkLpGxUboF}
}

@inproceedings{hayes2024measuring,
  title={Measuring memorization in language models via probabilistic extraction},
  author={Jamie Hayes and Marika Swanberg and Harsh Chaudhari and Itay Yona and Ilia Shumailov},
  booktitle={North American Chapter of the Association for Computational Linguistics},
  year={2024},
  url={https://api.semanticscholar.org/CorpusID:273638180}
}

@inproceedings{
biderman2023emergent,
title={Emergent and Predictable Memorization in Large Language Models},
author={Stella Biderman and USVSN Sai Prashanth and Lintang Sutawika and Hailey Schoelkopf and Quentin Gregory Anthony and Shivanshu Purohit and Edward Raff},
booktitle={Thirty-seventh Conference on Neural Information Processing Systems},
year={2023},
url={https://openreview.net/forum?id=Iq0DvhB4Kf}
}

@inproceedings{lee2022deduplicating,
    title = "Deduplicating Training Data Makes Language Models Better",
    author = "Lee, Katherine  and
      Ippolito, Daphne  and
      Nystrom, Andrew  and
      Zhang, Chiyuan  and
      Eck, Douglas  and
      Callison-Burch, Chris  and
      Carlini, Nicholas",
    editor = "Muresan, Smaranda  and
      Nakov, Preslav  and
      Villavicencio, Aline",
    booktitle = "Proceedings of the 60th Annual Meeting of the Association for Computational Linguistics (Volume 1: Long Papers)",
    month = may,
    year = "2022",
    address = "Dublin, Ireland",
    publisher = "Association for Computational Linguistics",
    url = "https://aclanthology.org/2022.acl-long.577/",
    doi = "10.18653/v1/2022.acl-long.577",
    pages = "8424--8445",
}

@inproceedings{ippolito2023preventing,
    title = "Preventing Generation of Verbatim Memorization in Language Models Gives a False Sense of Privacy",
    author = "Ippolito, Daphne  and
      Tramer, Florian  and
      Nasr, Milad  and
      Zhang, Chiyuan  and
      Jagielski, Matthew  and
      Lee, Katherine  and
      Choquette Choo, Christopher  and
      Carlini, Nicholas",
    editor = "Keet, C. Maria  and
      Lee, Hung-Yi  and
      Zarrie{\ss}, Sina",
    booktitle = "Proceedings of the 16th International Natural Language Generation Conference",
    month = sep,
    year = "2023",
    address = "Prague, Czechia",
    publisher = "Association for Computational Linguistics",
    url = "https://aclanthology.org/2023.inlg-main.3/",
    doi = "10.18653/v1/2023.inlg-main.3",
    pages = "28--53"
}

@article{duan2024membership,
  title={Do membership inference attacks work on large language models?},
  author={Duan, Michael and Suri, Anshuman and Mireshghallah, Niloofar and Min, Sewon and Shi, Weijia and Zettlemoyer, Luke and Tsvetkov, Yulia and Choi, Yejin and Evans, David and Hajishirzi, Hannaneh},
  journal={arXiv preprint arXiv:2402.07841},
  year={2024}
}

@misc{mimir_dataset,
  title = {MIMIR Dataset},
  author = {Duan, Michael and others},
  year = {2024},
  howpublished = {\url{https://huggingface.co/datasets/iamgroot42/mimir}},
  note = {Accessed: 2026-03-17}
}

@misc{realtimedata_arxiv_alltime,
  title = {RealTimeData arXiv AllTime Dataset},
  author = {RealTimeData},
  year = {2025},
  howpublished = {\url{https://huggingface.co/datasets/RealTimeData/arxiv_alltime}},
  note = {Accessed: 2026-03-17}
}

@misc{realtimedata_wikitext_alltime,
  title = {RealTimeData WikiText AllTime Dataset},
  author = {RealTimeData},
  year = {2025},
  howpublished = {\url{https://huggingface.co/datasets/RealTimeData/wikitext_alltime}},
  note = {Accessed: 2026-03-17}
}

@misc{llada_8b_base,
  title = {LLaDA-8B-Base},
  author = {GSAI-ML},
  year = {2025},
  howpublished = {\url{https://huggingface.co/GSAI-ML/LLaDA-8B-Base}},
  note = {Accessed: 2026-03-17}
}

@misc{llada_8b_instruct,
  title = {LLaDA-8B-Instruct},
  author = {GSAI-ML},
  year = {2025},
  howpublished = {\url{https://huggingface.co/GSAI-ML/LLaDA-8B-Instruct}},
  note = {Accessed: 2026-03-17}
}

@article{li2023chatdoctor,
  title={Chatdoctor: A medical chat model fine-tuned on a large language model meta-ai (llama) using medical domain knowledge},
  author={Li, Yunxiang and Li, Zihan and Zhang, Kai and Dan, Ruilong and Jiang, Steve and Zhang, You},
  journal={Cureus},
  volume={15},
  number={6},
  year={2023},
  publisher={Cureus}
}

@misc{alpaca,
  author = {Rohan Taori and Ishaan Gulrajani and Tianyi Zhang and Yann Dubois and Xuechen Li and Carlos Guestrin and Percy Liang and Tatsunori B. Hashimoto},
  title = {Stanford Alpaca: An Instruction-following LLaMA model},
  year = {2023},
  publisher = {GitHub},
  journal = {GitHub repository},
  howpublished = {\url{https://github.com/tatsu-lab/stanford_alpaca}}
}

@misc{lambert2024tulu,
      title={Tulu 3: Pushing Frontiers in Open Language Model Post-Training}, 
      author={Nathan Lambert and Jacob Morrison and Valentina Pyatkin and Shengyi Huang and Hamish Ivison and Faeze Brahman and Lester James V. Miranda and Alisa Liu and Nouha Dziri and Shane Lyu and Yuling Gu and Saumya Malik and Victoria Graf and Jena D. Hwang and Jiangjiang Yang and Ronan Le Bras and Oyvind Tafjord and Chris Wilhelm and Luca Soldaini and Noah A. Smith and Yizhong Wang and Pradeep Dasigi and Hannaneh Hajishirzi},
      year={2025},
      eprint={2411.15124},
      archivePrefix={arXiv}
}

@misc{Dream_v0_Instruct_7B,
  author = {Dream-org},
  title = {Dream-v0-Instruct-7B},
  year = {2024},
  howpublished = {\url{https://huggingface.co/Dream-org/Dream-v0-Instruct-7B}},
  note = {Accessed: 2026-05-24}
}

@misc{Dream_v0_Base_7B,
  author = {Dream-org},
  title = {Dream-v0-Base-7B},
  year = {2024},
  howpublished = {\url{https://huggingface.co/Dream-org/Dream-v0-Base-7B}},
  note = {Accessed: 2026-05-24}
}

@inproceedings{wang2026silent,
  title={Silent leaks: Implicit knowledge extraction attack on RAG systems},
  author={Wang, Yuhao and Qu, Wenjie and Zhai, Shengfang and Jiang, Yanze and Zichen, Liu and Liu, Yue and Dong, Yinpeng and Zhang, Jiaheng},
  booktitle={International Conference on Learning Representations},
  volume={2026},
  pages={24150--24191},
  year={2026}
}

@inproceedings{zhai2025efficient,
  title={Efficient input-level backdoor defense on text-to-image synthesis via neuron activation variation},
  author={Zhai, Shengfang and Li, Jiajun and Liu, Yue and Chen, Huanran and Tian, Zhihua and Qu, Wenjie and Shen, Qingni and Jia, Ruoxi and Dong, Yinpeng and Zhang, Jiaheng},
  booktitle={2025 IEEE/CVF International Conference on Computer Vision (ICCV)},
  pages={15182--15193},
  year={2025},
  organization={IEEE}
}

@article{zhai2026baddlm,
  title={BadDLM: Backdooring Diffusion Language Models with Diverse Targets},
  author={Zhai, Shengfang and Ji, Xiaoyang and Shi, Yuling and Gao, Haoran and Meng, Fanyu and Zeng, Yan and Fang, Yuejian and Dong, Yinpeng and Zhang, Jiaheng},
  journal={arXiv preprint arXiv:2605.09397},
  year={2026}
}

@article{zhu2024enhancing,
  title={Enhancing zero-shot vision models by label-free prompt distribution learning and bias correcting},
  author={Zhu, Xingyu and Zhu, Beier and Tan, Yi and Wang, Shuo and Hao, Yanbin and Zhang, Hanwang},
  journal={Advances in Neural Information Processing Systems},
  volume={37},
  pages={2001--2025},
  year={2024}
}

@article{zhai2024discovering,
  title={Discovering universal semantic triggers for text-to-image synthesis},
  author={Zhai, Shengfang and Wang, Weilong and Li, Jiajun and Dong, Yinpeng and Su, Hang and Shen, Qingni},
  journal={arXiv preprint arXiv:2402.07562},
  year={2024}
}

\newpage
\appendix
\section{Theoretical Analysis}
\subsection{Proof of \cref{th:pos_skew}}
\label{appendix:prf}
\begin{proof}[Proof of \cref{th:pos_skew}]
For brevity, we omit the token index $w_i$ for the single-epoch memory gain and denote it as $\delta$
Consider the random variable for a single epoch $\delta = m \cdot \phi(\beta)$ (subscripts omitted for brevity).
The marginal probability of a token being selected is $P(m=1) = \mathbb{E}[\beta] = 0.5$.
Conditioned on selection ($m=1$), the probability density of the mask rate becomes:
$$f_{\beta|m=1}(\beta) = 2\beta.$$
Let $V$ be a random variable drawn from this conditional distribution, representing the update magnitude when a token is selected, i.e., $V = \phi(\beta)|_{m=1}$. We can thus rewrite the unconditional update as $\delta = J \cdot V$, where the indicator $J \sim \text{Bernoulli}(0.5)$ is independent of $V$.

To determine skewness, we analyze the third central moment $\mu_3(\delta)$. Using the definition of moments for this mixture:
$$\mu_3(\delta) = \frac{1}{4} \left( 2\mathbb{E}[V^3] - 3\mathbb{E}[V]\mathbb{E}[V^2] + (\mathbb{E}[V])^3 \right)$$
Applying the AM-GM Inequality and Lyapunov’s Inequality: 
$$(\mathbb{E}[V^3])^{1/3} \ge (\mathbb{E}[V^2])^{1/2},$$
and given that $\phi(\beta)$ is non-constant (implying $\text{Var}(V) > 0$), we obtain the strict inequality:
\[
\begin{split}
& 2\mathbb{E}[V^3] + (\mathbb{E}[V])^3 > 3\mathbb{E}[V]\mathbb{E}[V^2]  \\ & \implies \mu_3(\delta) > 0 
\end{split}
\]

To extend this to the cumulative gain 
$\mathcal{G}_{w_{i}}= \sum_{k=1}^K \delta_k$, 
we note that within a short fine-tuning window, the model parameters evolve smoothly, i.e., $\theta_{k+1} \approx \theta_k$. 
Thus, the token selection mechanism, the mask-rate distribution, and the shape of the update function $\phi(\beta)$ do not drift significantly across nearby epochs. 
Based on this, we assume for analytical tractability that the token-level updates $\delta_k$ are locally approximately independent and identically distributed (locally approximately i.i.d.) across epochs within this window.
Then the skewness of the sum scales by $K^{-1/2}$, preserving the positive asymmetry:
$$
Skew(\mathcal{G}_{w_{i}}) = \frac{1}{\sqrt{K}}Skew(\delta) > 0.
$$

\end{proof}

Note that in this proof, we assume approximately i.i.d. token-level updates across epochs.
We further provide an analysis without this approximate i.i.d. assumption in Appendix \ref{appd:proof_wo_iid}.

\subsection{Long Tail Analysis}\label{appendix:longtail}
\begin{remark}[Mechanism of Tail Formation]
    While \cref{th:pos_skew} proves asymmetry, the specific "heavy-tailed" shape arises from the interaction between the sampling probability and the update strength. Consider the probability density function (PDF) $f_V(v)$ of the update magnitude $V$. By the change of variables, $f_V(v) = f_{\beta|m=1}(\phi^{-1}(v)) \cdot |\frac{d\beta}{dv}|$. Since $f_{\beta|m=1}(\beta) = 2\beta$, we have:$$f_V(v) \propto \beta(v) \cdot \left| \frac{d\beta}{dv} \right|$$Since $\phi(\beta)$ is monotonically decreasing, high memory gains (large $v$) correspond to low mask rates (small $\beta$). Crucially, as $v \to \infty$, $\beta \to 0$. This implies that the probability weight term $\beta(v)$ vanishes for large memory updates. Physically, this means extreme memorization events are structurally rare: they require a low $\beta$ to generate a large gradient, but low $\beta$ inherently reduces the probability of the token being sampled. This probabilistic suppression forces the density to decay, preventing the formation of a high-value mode (bimodality) and instead forming a monotonically decaying long tail.
\end{remark}

\begin{figure*}
    \centering
    \includegraphics[width=1\linewidth]{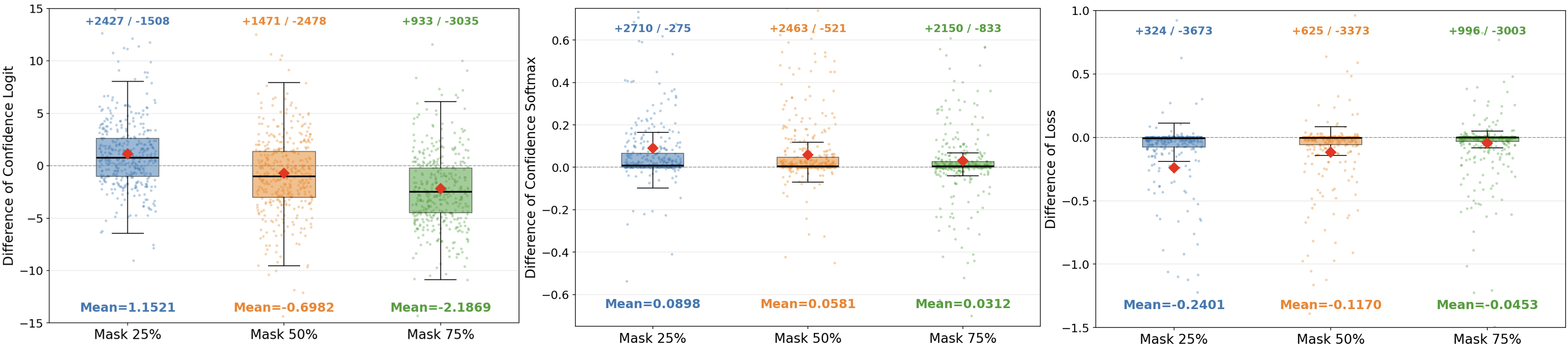}
    \caption{Inverse Mask-Ratio Scaling phenomenon with different metrics on XSUM~\cite{narayan2018don} datasets.}
    \label{fig:more_inverse}
\end{figure*}

\subsection{A Non-i.i.d. Extension of the Skewness Analysis}
\label{appd:proof_wo_iid}

The proof in Appendix~\ref{appendix:prf} uses a locally approximately i.i.d. approximation to transfer the positive skewness of single-epoch memory gains to the cumulative gain across epochs. This approximation leads to a clean analytical argument, but it is slightly stronger than what holds in practice. 
In this part, we relax the locally i.i.d. assumption and analyze the more general case in which epoch-wise token-level updates may be dependent. 
Although this setting no longer yields the same closed-form skewness scaling, the decomposition below shows why the positive-skewness tendency should generally persist unless cross-epoch dependence introduces a dominant negative third-order contribution.
More generally, for a fixed token $w_i$, let:
\[
\mathcal{G}_{w_i}=\sum_{k=1}^{K}\delta_{w_i,k},
\qquad
\widetilde{\delta}_{w_i,k}
=
\delta_{w_i,k}
-
\mathbb{E}[\delta_{w_i,k}].
\]
Without assuming independence across epochs, the third central moment of the cumulative gain can be written as:
\[
\mu_3(\mathcal{G}_{w_i})
=
\mathbb{E}
\left[
\left(
\sum_{k=1}^{K}\widetilde{\delta}_{w_i,k}
\right)^3
\right].
\]
Expanding this term gives:
\[
\mu_3(\mathcal{G}_{w_i})
=
\sum_{k=1}^{K}
\mathbb{E}
\left[
\widetilde{\delta}_{w_i,k}^{\,3}
\right]
+
C_{w_i,K},
\]
where
\[
\begin{aligned}
C_{w_i,K}
={}&
3
\sum_{\substack{1\leq a,b\leq K\\ a\neq b}}
\mathbb{E}\!\left[
\widetilde{\delta}_{w_i,a}^{\,2}
\widetilde{\delta}_{w_i,b}
\right] \\
&+
6
\sum_{1\leq a<b<c\leq K}
\mathbb{E}\!\left[
\widetilde{\delta}_{w_i,a}
\widetilde{\delta}_{w_i,b}
\widetilde{\delta}_{w_i,c}
\right].
\end{aligned}
\]
The first term captures the marginal third central moments of the epoch-wise gains. 
Under the per-epoch zero-inflated inverse mask-ratio mechanism, and under the finite third-moment and non-degeneracy assumptions stated in Appendix~\ref{appendix:prf}, these marginal third central moments are positive. 
The second term $C_{w_i,K}$ captures the mixed third-order central moments across epochs.

This decomposition clarifies the role of dependence across epochs. 
Dependence alone does not invalidate the argument. 
The conclusion would fail only if the epoch-wise gains show strong negative third-order dependence. 
We do not expect this to be common in our fine-tuning setting. 
The denoising objective and member examples stay the same across nearby checkpoints, while the mask pattern is resampled in each epoch. As a result, zero-inflated token selection and low-mask-ratio amplification are still the main sources of asymmetry.

Motivated by the finite fine-tuning window considered in our setting, we expect the mixed third-order contribution not to overwhelm the positive marginal contribution. 
This non-dominance relation can be written as:
\[
C_{w_i,K}
>
-
\sum_{k=1}^{K}
\mathbb{E}
\left[
\widetilde{\delta}_{w_i,k}^{\,3}
\right],
\]
which implies
\(
\mu_3(\mathcal{G}_{w_i})>0.
\)
Since $\operatorname{Var}(\mathcal{G}_{w_i})>0$ under the non-degeneracy assumption, we have:
\[
\operatorname{Skew}(\mathcal{G}_{w_i})
=
\frac{
\mu_3(\mathcal{G}_{w_i})
}{
\operatorname{Var}(\mathcal{G}_{w_i})^{3/2}
}
>0.
\]

\section{Validation of Inverse Mask-Ratio Scaling Assumption}
\label{sec:more_inverse}

We use different metrics to estimate the memorization of individual tokens, including prediction confidence, with or without softmax, and token-level loss.
We present the experimental results using violin plots with jitter plots (\cref{fig:more_inverse}).
It is evident that a lower mask ratio leads to more significant token memorization on average under the same number of training epochs, across different metrics.

\begin{figure*}[t]
    \centering
    \includegraphics[width=0.99\linewidth]{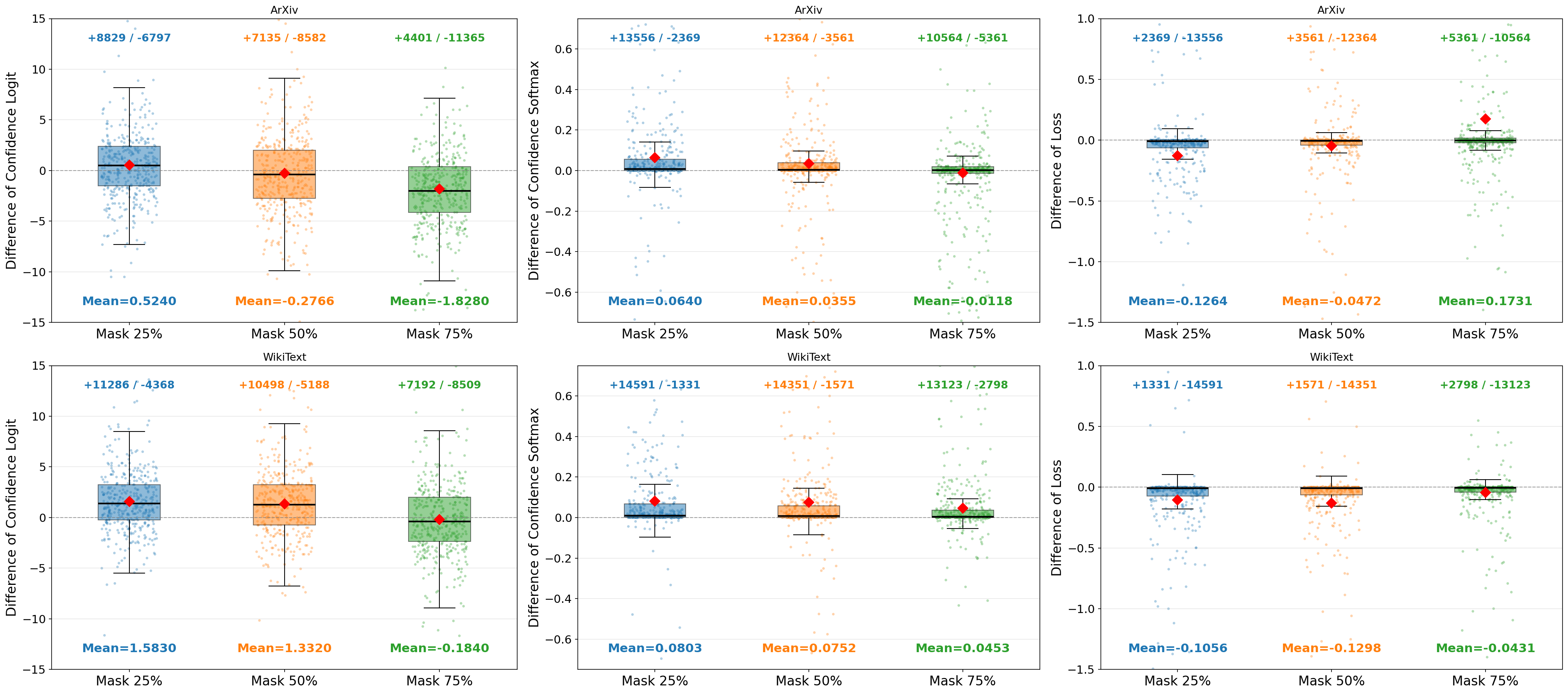}
    \caption{Validation under general domain: the inverse mask-ratio scaling phenomenon under different metrics on larger-scale ArXiv~\cite{realtimedata_arxiv_alltime} and WikiText~\cite{realtimedata_wikitext_alltime} datasets.
}
    \label{fig:arxiv_wiki_logit_softmax_loss}
\end{figure*}

To further validate the broad applicability of the inverse mask-ratio scaling assumption, we further extend our experiments to different data domains, including ArXiv~\cite{realtimedata_arxiv_alltime} and WikiText~\cite{realtimedata_wikitext_alltime}, and increase the dataset size from 125 samples to 1,000 samples. As shown in \cref{fig:arxiv_wiki_logit_softmax_loss}, the experimental results show that this assumption widely holds across different data domains and data scales.


\section{Implementation of MIA Baselines}
\label{appd:baselines}

To the best of our knowledge, SAMA \cite{chen2026membership} is currently the only MI method specifically designed for diffusion language models.
Therefore, we further include AR-LM MI methods and MI methods for vision diffusion models that can be transferred to DLMs as baselines under the same threat models (\cref{subsec:threat}). 
Overall, our baselines include three categories: (1) AR-LM MI baselines: Loss \cite{yeom2018privacy}, Min-K\%\cite{shi2023detecting}, Min-K\%++\cite{zhang2024min}, and Calibration \cite{watson2021importance}; (2) diffusion vision model MI baselines: SecMI \cite{duan2023diffusion}; and (3) DLM-specific baselines: SAMA \cite{chen2026membership}.

\textbf{To ensure fairness, all baselines are evaluated under the same settings and computing budget}, including the same member and non-member splits, the same target and reference access, the same ROC-based metrics, and the same denoising-step budget of 16 forward passes for each sample.
\ding{182} For the DLM-specific baseline, we strictly follow the settings in its original paper.
\ding{183} For baselines adapted from AR-LMs or vision diffusion models, we adapt them to DLMs as closely as possible.
\ding{184} For SAMA \cite{chen2026membership} we strictly follow the original settings in the paper.
Below, we describe how we implement each baseline.

\subsection{AR-LM MI Baselines Adapted for DLMs}

\paragraph{Loss \cite{yeom2018privacy}.}  We use the same denoising ratio as our method to compute the DLM loss. We then average the loss over 16 denoising runs, which means that the loss is computed 16 times. A smaller average loss indicates that the sample is more likely to belong to the member set.

\paragraph{Min-K\%~\cite{shi2023detecting} and Min-K\%++ \cite{zhang2024min}.} Following the original papers, we compute the Min-K\% and Min-K\%++ scores from token-level log-likelihoods at each denoising step, using the same denoising ratio as our method. For Min-K\%, we average the log-likelihoods of the bottom 20\% tokens. For Min-K\%++, we first apply vocabulary-level normalization and then average the bottom 20\% normalized token scores. Finally, we average the score over 16 denoising runs. A higher score indicates that the sample is more likely to belong to the member set.

\paragraph{Calibration \cite{watson2021importance}.} Following the original paper, we use a reference model to calibrate sample difficulty. Specifically, we compute the difference between the average loss of the reference model and that of the target model on masked tokens, averaged over 16 denoising runs, and use this difference as the Calibration score. A larger score means that the target model has a lower loss than the reference model on the sample, so the sample is more likely to belong to the member set.

\subsection{(Vision) Diffusion MI Baselines adapted for DLMs}

\paragraph{SecMI~\cite{duan2023diffusion}.}Following the original paper, we strictly compute the prior difference between steps (t+1) and (t). A lower reconstruction error indicates a higher membership probability, because the model can better reconstruct the tokens it has memorized during training. In the original paper, $t$ is set to 100 steps out of 1000. For DLMs, we set the masking ratio to 10\% to remain consistent with the original paper.

\paragraph{PIA~\cite{kong2024an}.} We do not adapt PIA because this method relies on a derived formula that extrapolates the result as $t$ approaches 0. For DLMs, this derivation does not hold.

\paragraph{CLiD~\cite{zhai2024membership}.}
We do not adapt CLiD because it is designed for conditional vision diffusion models that take both text and images as input, which does not match the setting of DLMs.

\section{Dataset Selection} 
\label{appd:dataset}

\subsection{Hallucination of MI Success on MIMIR Dataset in Previous Works.}
\label{sec:hallu}
\begin{table}[t]

  \centering

  \caption{Performance of Loss MI on unfinetuned models with MIMIR datasets. 
  The high AUC values indicate that \textbf{LLaDA model possesses a strong prior on the MIMIR training (member) set}, which leads to a hallucination of membership inference success.
  }
  \resizebox{\linewidth}{!}{

    \begin{tabular}{lccc}

    \toprule

    Dataset & AUC   & TPR@10\%FPR & TPR@1\%FPR \\

    \midrule

    Arxiv           & 0.70  & 0.32  & 0.05 \\

    Github          & 0.81  & 0.47  & 0.07 \\

    Hackernews      & 0.56  & 0.18  & 0.02 \\

    Pile\_cc        & 0.53  & 0.13  & 0.02 \\

    Pubmed\_central & 0.65  & 0.22  & 0.01 \\

    Wikipedia\_(en) & 0.63  & 0.28  & 0.05 \\

    \midrule

    \textbf{Average} & \textbf{0.64} & \textbf{0.27} & \textbf{0.04} \\

    \bottomrule

    \end{tabular}
}\label{tab:pre_mimir}
\end{table}
Note that the evaluation of membership inference (MI) on fine-tuned models differs fundamentally from that on pre-trained models~\cite{duan2024membership}. 
For MI on pre-trained language models, researchers perform no additional training and utilize data from distinct release periods \cite{mimir_dataset} to detect whether the model memorized specific samples. 
In contrast, 
evaluation data for MI during the fine-tuning stage must ensure that the model possesses equivalent prior knowledge for both member and non-member samples. 
This consistency ensures that MI success stems from actual memory extraction rather than internal model priors. 
We note that a concurrent work \cite{chen2026membership} fine-tunes LLaDA-8B-Base on MIMIR \cite{mimir_dataset} to evaluate its MI performance. 
However, MIMIR is specifically designed for pre-training MI analysis. 
We evaluate the unfinetuned LLaDA on the MIMIR dataset using a loss-based MI method and report the results in \cref{tab:pre_mimir}. 
We find that the model exhibits strong inherent priors on MIMIR even without any fine-tuning. This bias suggests that seemingly successful MI results may be independent of the inference method itself, making MIMIR unsuitable for evaluating MI during the fine-tuning phase.

\subsection{Our Dataset Selection}
To ensure a fair evaluation, we utilize recently collected data \cite{realtimedata_arxiv_alltime,realtimedata_wikitext_alltime} for the Arxiv and Wiki datasets to guarantee that the model has no prior knowledge of the training or test sets. 
For the other dataset in \cref{sec:evaluation}, we clean the samples to prevent leakage between member and non-member sets and to ensure that model internal priors or bias \cite{zhu2024enhancing} remain consistent.

\section{Justification of Accessing Reference Models in Logit-based Settings.}
\label{appd:justify}

In our paper, we follow the standard setting in membership inference \cite{yeom2018privacy, shi2023detecting,zhang2024min,duan2023diffusion,kong2024an,zhai2024membership,chen2026membership,fu2024membership}, where the attacker can access the logits of the target model  (known as logit-based methods).
In addition, we assume that the adversary can access a reference model, although it may not be fully aligned with the target model, following \cite{chen2026membership,fu2024membership}. 

In this part, we justify the scenario where the attacker can access a reference model under the logit-based threat model: 
(1) Knowing the architecture of the target model is natural under the logit-based setting. Models with different architectures usually employ different tokenizers, output dimensions, and token representations. For example, LLaDA~\cite{nie2025large} and Dream~\cite{ye2025dream} use completely different tokenizers and token representations. Therefore, under the logit-based setting, an MI adversary can naturally infer the architecture of the target model and then select the corresponding reference model.
(2) Given the model architecture, it is easy to obtain either the non fine-tuned base model or a structurally identical but misaligned reference model. Since our method targets the fine-tuning stage, the corresponding base models are usually open-source. Therefore, once the MI adversary knows the architecture of the target model, it can easily obtain the corresponding non fine-tuned base model. Moreover, our experiments show that our method remains effective even when the reference model parameters are perturbed through additional fine-tuning (\cref{sec:weaker}).

\section{Additional Results on PII Extraction}
\label{appendix:more-results-pii-extraction}

\subsection{Ablation Study}
\label{appendix:recon-abl}

\begin{figure*}[t]
  \centering
  \begin{subfigure}[b]{0.24\linewidth}
    \includegraphics[width=\linewidth]{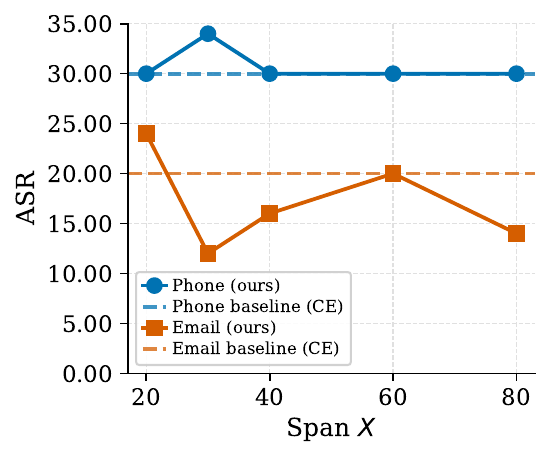}
  \end{subfigure}
  \hfill
  \begin{subfigure}[b]{0.24\linewidth}
    \includegraphics[width=\linewidth]{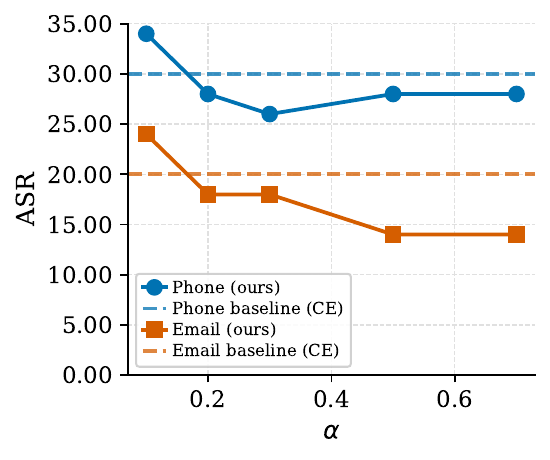}
  \end{subfigure}
  \hfill
  \begin{subfigure}[b]{0.24\linewidth}
    \includegraphics[width=\linewidth]{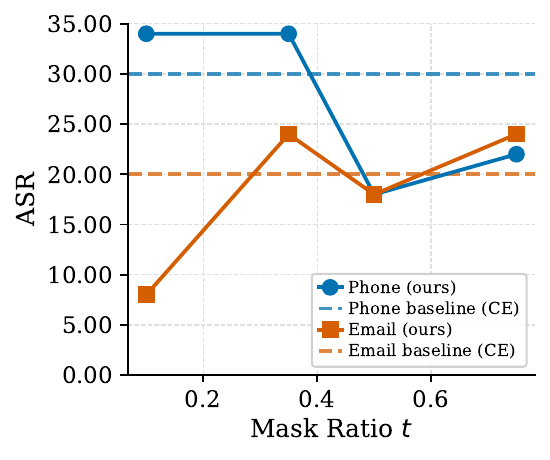}
  \end{subfigure}
  \hfill
  \begin{subfigure}[b]{0.24\linewidth}
    \includegraphics[width=\linewidth]{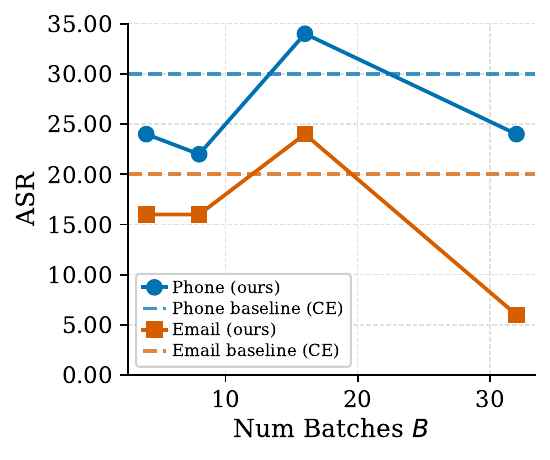}
  \end{subfigure}
  \caption{Ablation results of LLaDA for PII extraction.
           Dashed lines denote the CE-based (\texttt{Ranking}) method.}
  \label{fig:llada_ablation}
\end{figure*}

\topic{Hyperparameter Investigation.} We investigate four hyperparameters affecting our combined loss-skewness score and report the resulting ASR for LLaDA in Fig.~\ref{fig:llada_ablation}. From left to right, the plots vary: (1) the token span $X$ used for skewness calculation, (2) the weighting parameter $\alpha$, (3) the masking ratio $t$, and (4) the number of masking batches $B$. We observe only a weak relationship between token span $X$ and ASR, although emails show some improvement between spans of 20 and 30. Increasing $\alpha$ consistently reduces ASR across both PII types, highlighting the importance of the cross-entropy loss signal in the combined score. For mask ratio $t$, the effect is strongly PII-dependent. Higher mask ratios reduce ASR for phone numbers, but substantially improve performance for emails compared to lower ratios such as $0.1$. Finally, for the number of masking batches $B$, both very small and very large values degrade performance. We hypothesize that excessively large $B$ values reduce the skewness of loss differences, weakening the membership signal.

\begin{table}[ht]
  \centering
  \caption{%
    ASR@k (\%) of different attack methods on LLaDA-8B-Base
    for phone numbers and email addresses.
    Higher is better. Best results per column are in bold.
  }
  \resizebox{\linewidth}{!}{
    \begin{tabular}{llcccccc}
    \toprule
    \multirow{2}{*}{\textbf{Our Attacks}} 
    & \multirow{2}{*}{\textbf{Signal}} 
    & \multicolumn{3}{c}{\textbf{Phone}} 
    & \multicolumn{3}{c}{\textbf{Email}} \\ \cmidrule(lr){3-5} \cmidrule(lr){6-8}
    & & \textbf{@1} & \textbf{@3} & \textbf{@5} 
      & \textbf{@1} & \textbf{@3} & \textbf{@5} \\ 
    \midrule \midrule

    \texttt{Ranking} & CE 
    & 30 & \textbf{42} & \textbf{44} 
    & 20 & 36 & 42 \\

    \texttt{Mask only PII} & \shortname 
    & 22 & 36 & 42 
    & 16 & 28 & 34 \\

    \texttt{Mask only context} & \shortname 
    & 10 & 22 & 26 
    & \hspace{0.5em}8 & 10 & 14 \\

    \texttt{Combine-0.1} & Both 
    & \textbf{34} & 38 & 42 
    & \textbf{24} & \textbf{44} & \textbf{52} \\

    \bottomrule
    \end{tabular}
  }
  \label{tab:pii_recon_topk}
\end{table}

\topic{Top-K candidate Ranking}. In Table \ref{tab:pii_recon_topk}, we show the top-k accuracy of different ranking methods where k is either 1, 3, or 5. Surprisingly for phone numbers, \texttt{Ranking} narrowly beats \texttt{Combine-0.1} for values of k larger than 1. Interestingly, the top-k accuracy advantage over \texttt{Ranking} in the case of emails improves drastically when k is equal to 5, with over 50\% accuracy, the highest of any method and any value of k. Unsurprisingly, \texttt{Mask only PII} and \texttt{Mask only context} perform worse by comparison, similar to the top-1 accuracy.

\topic{Setting $\alpha$.} We assume that an attacker with partial knowledge of the fine-tuning dataset could reasonably approximate an effective $\alpha$ for the loss-skewness combined attack.

\begin{figure}[h]
  \centering
  \begin{subfigure}[b]{0.8\linewidth}
    \includegraphics[width=\linewidth]{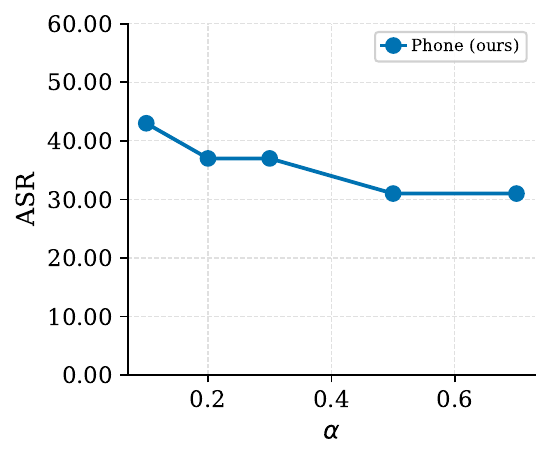}
  \end{subfigure}
  \hfill
  \caption{
  Impact of choosing $\alpha$ in (0, 1) on LLaDA for phone PII reconstruction,
  showing strong performance even with a small hold-out set. Higher is better.}
  \label{fig:alpha_justification}
\end{figure}
In this scenario, we assume that the attacker has access to a small set of full, un-redacted records from the training dataset.
In addition, we assume that they can interact with the target model the same way as when they target the unknown PII. 
To simulate this attacker, we select 16 records containing phone PII, and do a sweep of the $\alpha$ hyperparameter across \texttt{Combine} attacks. 
Critically, these 16 records form a hold-out set that does not overlap with the set of target records and PIIs used in the main evaluation. 
Fig. \ref{fig:alpha_justification} shows that with only a small hold-out sample of target records, an $\alpha$ value of 0.1 leads to the highest ASR.

\begin{figure}[ht]
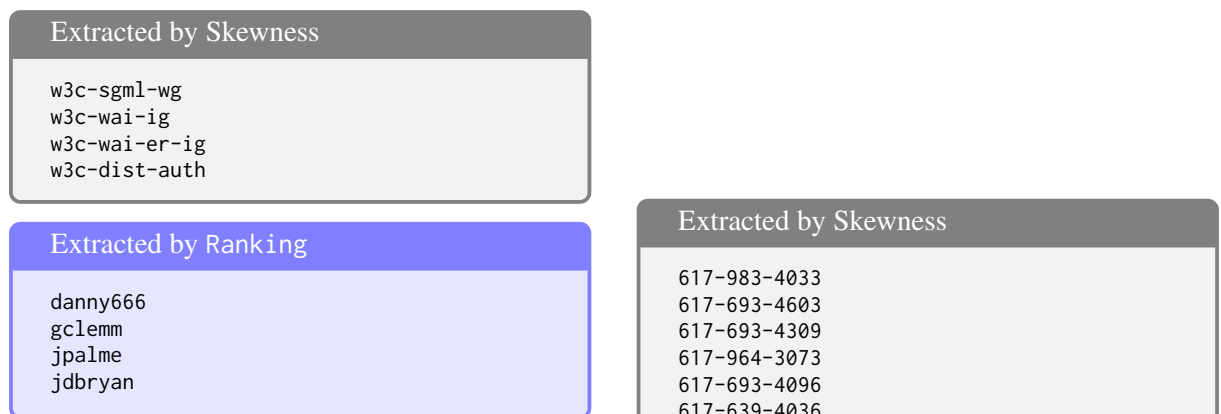

\centering
\begin{tcolorbox}[title=Extracted by Skewness, colback=gray!10, colframe=black!50]
\small
\begin{verbatim}
w3c-sgml-wg
w3c-wai-ig
w3c-wai-er-ig
w3c-dist-auth
\end{verbatim}
\end{tcolorbox}

\begin{tcolorbox}[title=Extracted by \texttt{Ranking}, colback=blue!10, colframe=blue!50]
\small
\begin{verbatim}
danny666
gclemm
jpalme
jdbryan
\end{verbatim}
\end{tcolorbox}
\caption{%
    Examples of email IDs extracted exclusively by each candidate-scoring method.}
\label{fig:qualitative-pii-structure}
\end{figure}

\begin{figure}[ht]
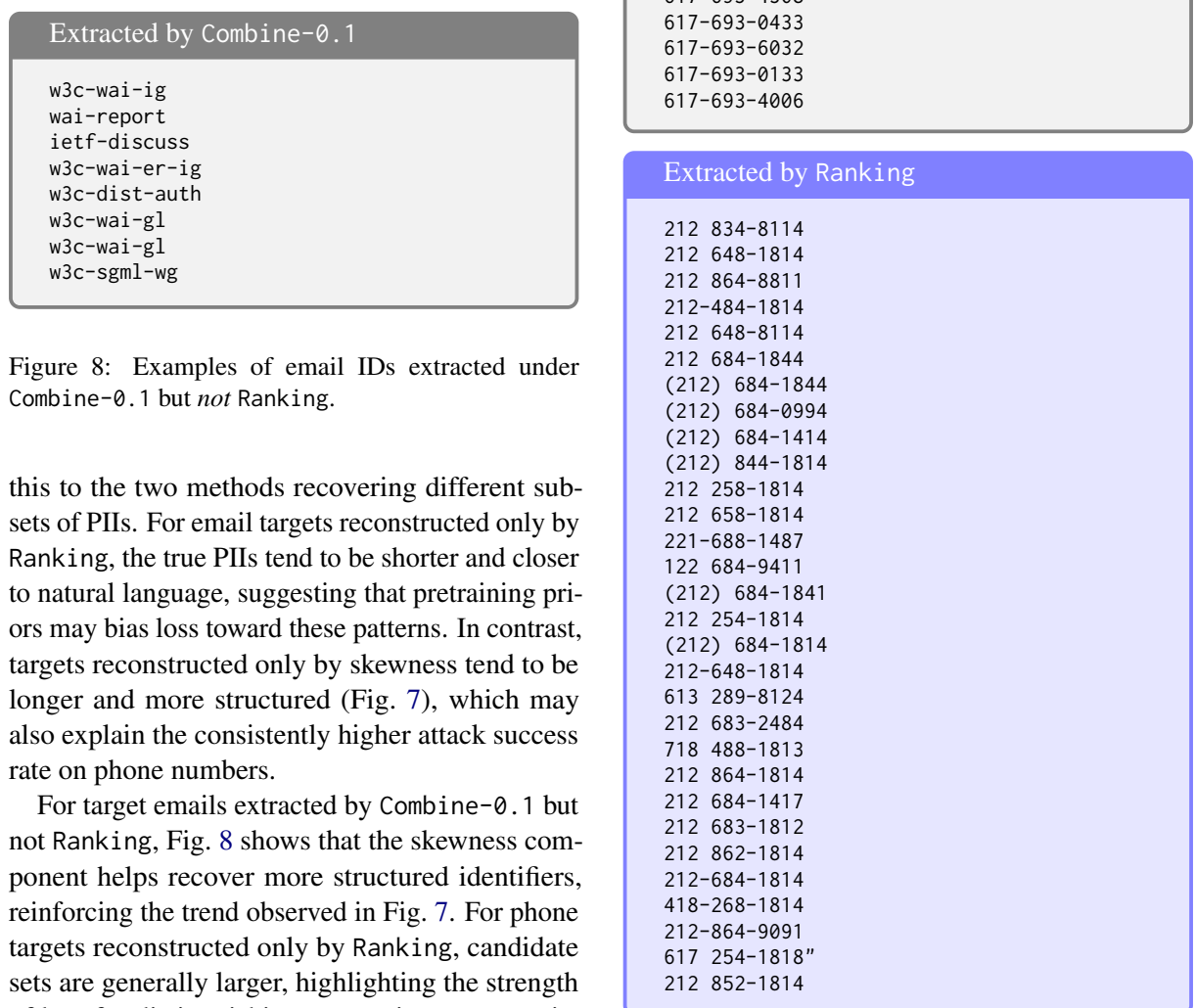

\centering
\begin{tcolorbox}[title=Extracted by \texttt{Combine-0.1}, colback=gray!10, colframe=black!50]
\small
\begin{verbatim}
w3c-wai-ig 
wai-report
ietf-discuss
w3c-wai-er-ig
w3c-dist-auth
w3c-wai-gl
w3c-wai-gl
w3c-sgml-wg
\end{verbatim}
\end{tcolorbox}

\caption{Examples of email IDs extracted under \texttt{Combine-0.1} but \textit{not} \texttt{Ranking}.
}
\label{fig:qualitative-pii-structure-hybrid}
\end{figure}

\subsection{Qualitative analysis.} 
\label{appendix:pii-qual-analysis}

Our \texttt{Combine-0.1} score achieves higher top-1 accuracy across multiple settings, although skewness alone never outperforms \texttt{Ranking}. We attribute this to the two methods recovering different subsets of PIIs. For email targets reconstructed only by \texttt{Ranking}, the true PIIs tend to be shorter and closer to natural language, suggesting that pretraining priors may bias loss toward these patterns. In contrast, targets reconstructed only by skewness tend to be longer and more structured (Fig.~\ref{fig:qualitative-pii-structure}), which may also explain the consistently higher attack success rate on phone numbers. 

For target emails extracted by \texttt{Combine-0.1} but not \texttt{Ranking}, Fig.~\ref{fig:qualitative-pii-structure-hybrid} shows that the skewness component helps recover more structured identifiers, reinforcing the trend observed in Fig.~\ref{fig:qualitative-pii-structure}. For phone targets reconstructed only by \texttt{Ranking}, candidate sets are generally larger, highlighting the strength of loss for distinguishing correct in-context training data from a large non-training set. In contrast, skewness-only reconstructions tend to produce smaller but more token-diverse candidate sets (Fig.~\ref{fig:qualitative-candidate-output}). These skewness-only targets are also typically longer, aligning with prior observations that longer contexts increase memorization in autoregressive models \cite{carlini2023quantifying}.

\begin{figure}[ht]
\centering
\begin{tcolorbox}[title=Extracted by Skewness, colback=gray!10, colframe=black!50]
\small
\begin{verbatim}
617-983-4033 
617-693-4603 
617-693-4309 
617-964-3073 
617-693-4096 
617-639-4036 
617-693-4063
617-693-3034
617-693-4039 
617-693-4030 
617-693-4308 
617-693-0433 
617-693-6032 
617-693-0133 
617-693-4006
\end{verbatim}
\end{tcolorbox}

\begin{tcolorbox}[title=Extracted by \texttt{Ranking}, colback=blue!10, colframe=blue!50]
\small
\begin{verbatim}
212 834-8114 
212 648-1814 
212 864-8811 
212-484-1814 
212 648-8114 
212 684-1844 
(212) 684-1844
(212) 684-0994
(212) 684-1414
(212) 844-1814
212 258-1814
212 658-1814
221-688-1487
122 684-9411
(212) 684-1841 
212 254-1814 
(212) 684-1814 
212-648-1814 
613 289-8124 
212 683-2484 
718 488-1813 
212 864-1814 
212 684-1417 
212 683-1812 
212 862-1814 
212-684-1814 
418-268-1814 
212-864-9091 
617 254-1818" 
212 852-1814
\end{verbatim}
\end{tcolorbox}
\caption{Examples of candidate sets for targets extracted exclusively by a given scoring method.}
\label{fig:qualitative-candidate-output}
\end{figure}

\end{document}